\documentclass{article}

 \PassOptionsToPackage{numbers, compress}{natbib}
\usepackage[final]{neurips_2024}

\usepackage{color}
\usepackage[dvipsnames]{xcolor}
\usepackage{colortbl}
\usepackage{graphicx}
\usepackage{booktabs}
\usepackage{xspace}
\usepackage[utf8]{inputenc} 
\usepackage[T1]{fontenc}    
\usepackage[colorlinks,
            linkcolor=VioletRed,      
            anchorcolor=VioletRed, 
            citecolor=VioletRed,       
            ]{hyperref}
\usepackage{subfigure}
\usepackage{tcolorbox}
\usepackage{wrapfig}
\usepackage{url}            
\usepackage{booktabs}       
\usepackage{amsfonts}       
\usepackage{nicefrac}       
\usepackage{microtype}      
\usepackage{xcolor}         
\usepackage{bm}
\usepackage{graphicx}
\usepackage{amsmath}
\usepackage{cleveref}
\usepackage{amsthm}
\usepackage{tikz}
\definecolor{Gray}{gray}{0.85}
\newcommand{\Gray}[0]{\rowcolor{gray!20}}

\definecolor{sclgreyblue}{rgb}{0.2,0.3,0.5}%
\newcommand{\lightgrey}[1]{\texttt{\color{black!50} #1}}
\newcommand{\abbr}[0]{SliME\xspace}

\def \L {\mathcal{L}}

\def \R {\mathbb{R}}

\newtheorem{thm}{Theorem}

\newtheorem{lemma}{Lemma}

\title{Beyond LLaVA-HD: Diving into High-Resolution Large Multimodal Models}

\author{Yi-Fan Zhang$^{1,2}$, ~ Qingsong Wen$^{3}$, ~Chaoyou Fu, ~Xue Wang$^4$,
 \\ \textbf{Zhang Zhang$^{1,2}$,~ Liang Wang$^{1,2}$,~ Rong Jin$^{5}$} \\
$^{1}$State Key Laboratory of Multimodal Artificial Intelligence Systems (MAIS), Institute of Automation\\
$^{2}$School of Artificial Intelligence, University of Chinese Academy of Sciences (UCAS) \\
$^{3}$Squirrel AI Learning; $^{4}$Alibaba Group; $^{5}$Meta AI \\
\url{https://github.com/yfzhang114/SliME}
} %

\begin{document}

\maketitle
\begin{tikzpicture}[remember picture,overlay,shift={(current page.north west)}]
\node[anchor=north west,xshift=1.8cm,yshift=-2.85cm]{\scalebox{-1}[1]{\includegraphics[width=2.cm]{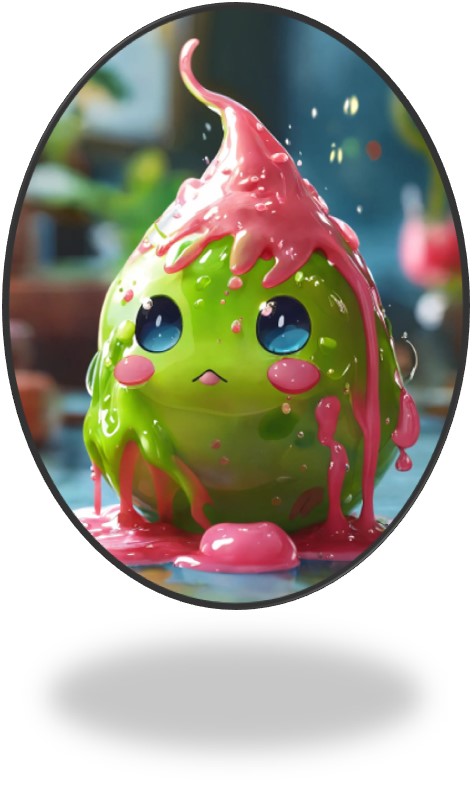}}};
\vspace{-0.6cm}
\end{tikzpicture}
\vspace{-0.6cm}
\begin{abstract}

Seeing clearly with high resolution is a foundation of Large Multimodal Models (LMMs), which has been proven to be vital for visual perception and reasoning. Existing works usually employ a straightforward resolution upscaling method, where the image consists of global and local branches, with the latter being the sliced image patches but resized to the same resolution as the former. This means that higher resolution requires more local patches, resulting in exorbitant computational expenses, and meanwhile, the dominance of local image tokens may diminish the global context. In this paper, we dive into the problems and propose a new framework as well as an elaborate optimization strategy. Specifically, we extract contextual information from the global view using a mixture of adapters, based on the observation that different adapters excel at different tasks. With regard to local patches, learnable query embeddings are introduced to reduce image tokens, the most important tokens accounting for the user question will be further selected by a similarity-based selector. Our empirical results demonstrate a `less is more' pattern, where \textit{utilizing fewer but more informative local image tokens leads to improved performance}. Besides, a significant challenge lies in the training strategy, as simultaneous end-to-end training of the global mining block and local compression block does not yield optimal results. We thus advocate for an alternating training way, ensuring balanced learning between global and local aspects. Finally, we also introduce a challenging dataset with high requirements for image detail, enhancing the training of the local compression layer. The proposed method, termed LMM with \textbf{S}ophisticated Tasks, \textbf{L}ocal \textbf{i}mage compression, and \textbf{M}ixture of global \textbf{E}xperts (\abbr), achieves leading performance across various benchmarks with only 2 million training data.
\end{abstract}

\section{Introduction}\label{sec:intro}
In the past years, we are fortunate to witness a great flourish of LMMs~\citep{awadalla2023openflamingo, dai2024instructblip, liu2023visual}. 
However, they still struggle with complex visual perception~\citep{zhang2024debiasing, huang2023opera} and reasoning tasks~\citep{li2024mini, yin2023survey}. 
Empirical studies have shown that employing higher resolutions is a good solution~\citep{bai2023qwen,liu2023improved,li2023monkey,mckinzie2024mm1}. Approaches like LLaVA-Next~\citep{liu2024llavanext} segment high-resolution images into multiple patches, encoding each one independently before concatenating all local patch tokens with the original global image tokens, albeit at an escalated computational cost. 
The other models like Monkey~\citep{li2023monkey} and LLaVA-UHD~\citep{xu2024llava-uhd} also split images into patches, but subsequently compress them to avoid redundant tokens. 
In such cases, for high-resolution images, the local image tokens dominate the feature space. 
For example, in a $1024\times 1024$ image divided into $9$ patches, the global image token accounts for only $1/10$. 

In contrast, our core idea posits that global information should be prioritized, thus we aim to extract and retain as much global context as possible while enhancing it with local image details. 
In this study, we initially segment the images into patches according to their resolution. 
The image tokens are then categorized into two groups: the global view and local patches. 
For the former, we preserve the token count to retain all contextual information and utilize a mixture of adapters to further explore global context. 
As displayed in Fig.~\ref{fig:router} (b), we employ a Multilayer Perceptron (MLP) to project image features into the feature space of the LLM, and a set of learnable queries named qformer are employed to extract crucial global information. 
Softly mixing outputs from the two adapters aids the LLM in comprehending the global context more effectively. 
Considering the local patches, they provide additional image details but are compressed using a querying transformer to mitigate computational costs. 
As shown in Fig.~\ref{fig:router} (d), a text-guided router is further proposed to select the most relevant local image tokens corresponding to the input instruction or question, thereby avoiding excessive image tokens and focusing on pertinent image information.

At the same time, we find that it is challenging to train the global projection and local compression simultaneously.
The simplicity of the projection layer makes it easy to train, but also causes the model to degenerate rapidly due to over-reliance on global features and neglect of local counterparts. 
We formalize this as a bi-linear problem and theoretically show that simultaneously updating both blocks does not converge to optimal results. Instead, we propose to alternatively train the global projection block and local compression block, by which we ensure that both global and local features are effectively learned and utilized.

The training is known to be data-driven.
The current data instances primarily stem from real-world captions, general QA, and a limited number of real-world conversations sampled from robust LLMs. 
Most of these instances revolve around basic perception, recognition, and reasoning tasks, such as understanding relationships among objects.
However, there are two notable flaws: firstly, the tasks are not challenging enough and largely lack intricate visual reasoning tasks; secondly, many questions only pertain to specific objects or actions, neglecting the need for all image details. 
This limitation hampers the full utilization of the capabilities offered by our high-resolution framework. 
To this end, this paper meticulously gathers and filters datasets to create the Science and Mathematical Reasoning dataset (SMR), which encompasses nine challenging tasks spanning natural science, mathematical problems, and scientific chart comprehension. 
Some of these tasks provide complete reasoning paths, compelling the model to articulate the entire reasoning process. 
Importantly, many images in the SMR dataset contain rich annotations. 
Completing such intricate reasoning tasks necessitates a thorough understanding of image details, which will greatly benefit the training of our framework.


LMM with \textbf{S}ophisticated Tasks, \textbf{L}ocal \textbf{i}mage augmentation, and \textbf{M}ixture of global \textbf{E}xperts (\abbr) can be readily instantiated with a range of LLMs. Extensive empirical studies validate the effectiveness of our proposed method. Remarkably, our approach achieves leading performance across various settings, even matching the performance of well-established models such as Gemini Pro~\citep{team2023gemini} and Qwen-VL-Plus~\citep{bai2023qwen} in about 10 benchmarks with only 8B LLM and 2 million data. These results underscore the potential of \abbr to set new benchmarks, highlighting its advanced capabilities.

\section{Method}

We delineate our method aimed at enhancing LMMs' image understanding capabilities in this section. We utilize adaptive slicing to scale input resolution, and refine global context via a soft mixture of experts. Additionally, we compress local features using a query transformer architecture~\footnote{In this context, the abbreviation 'qformer' refers to query former, where we utilize learnable query embeddings as described in previous works~\citep{awadalla2023openflamingo,jaegle2021perceiver}, rather than employing the Qformer approach~\citep{dai2024instructblip}}, select features optimally with a text-guided router, and employ an alternating training scheme to optimize the bilinear optimization problem. These strategies collectively improve both the computational efficiency and performance of LMMs. Furthermore, we introduce the SMR dataset, known for its challenging nature and high demand for understanding image details, making it an ideal choice for training high-resolution frameworks.


\subsection{Refining Global Context with a Soft Mixture of Experts}
\textbf{Scaling Input Resolution by Adaptive Slicing.} Initially, we explore various grid options for slicing images, similar to LLaVA-Next, but with finer granularity (see Fig.~\ref{fig:slcing}). We investigate resolutions ranging from $336 \times (m,n)$ with $m=1,n=1$ to $m=6,n=6$ to determine the most efficient option. To provide a global context, we pad and resize the image to a uniform size of $336\times 336$ and concatenate it with local features. For images with shapes $W$ and $H$, we iterate through all available partition strategies. For instance, when using the strategy $m*n$, the resize scale can be calculated as $s=\min\{m*336/W, n*336/H\}$. The utilized resolution after scaling will be $\min\{W*H, W*s*H*s\}$, and the wasted resolution will be $m*336*n*336-\min\{W*H, W*s*H*s\}$. We select the best partition by maximizing the utilized resolution and minimizing the wasted resolution when the utilized resolution is the same.

\textbf{Why not Compress Global Image Tokens for Efficiency?} Our approach is inspired by empirical observations, consistent with previous findings~\citep{zeng2023matters}: when employing attention-based models as adapters to reduce tokens or bridge the modality gap, a more intricate hyper-parameter search may be required to achieve performance comparable to simpler MLP. As depicted in Fig.~\ref{fig:teaser_alt}, replacing the MLP adapter of LLaVA-v1.5 with the query former of the same number of tokens yields significantly inferior performance on most benchmarks. A simpler projector compels the LLM to better understand visual inputs, leading to enhanced generalization~\citep{lin2023vila}. Consequently, we refrain from reducing token numbers for global images and instead preserve all global information through simple projection.

\textbf{Global Context Refinement by Soft Mixture of Experts.} Although query former is inferior to MLP on most benchmarks, the learnable query embeddings and attention mechanism allow for a different feature selection strategy, and on some benchmarks such as ScienceQA (SQA)~\citep{lu2022learn}, query former achieves better performance. Building on the insights from our analysis, we propose a novel approach to refine global context features by leveraging the strengths of both MLP and query former adapters. Specifically, we employ a noisy Mixture of Experts (MOE) framework to combine the benefits of these two types of frameworks.
In this framework, for feature $x$ from the vision encoder, a learned gating network $G$ determines the weights for two adapters: $G(x)_0 f_m(x) + G(x)_1 f_q(x)$. The gating network learns to dynamically adjust the importance of each adapter based on the input feature. To prevent the gating network from predominantly activating the same adapter, we introduce learnable noise during training\footnote{As mentioned earlier, the training of MLP is easy to converge quickly, potentially causing the gated network to assign higher weight to MLP, hindering full training of query former.}. This is achieved through the following equation: $G(x)=\text{Softmax}\left(\{(x\cdot W_g)_i + \text{Normal}(0,1)\cdot\text{Softplus}\left((x\cdot W_{noise})_i\right)\}_{i=1}^2\right)$.

\subsection{Local Feature Mining with Compression and Selection}
\begin{wrapfigure}{r}{0.42\linewidth}
\vspace{-0.8cm}
  \begin{center}
    \includegraphics[width=\linewidth]{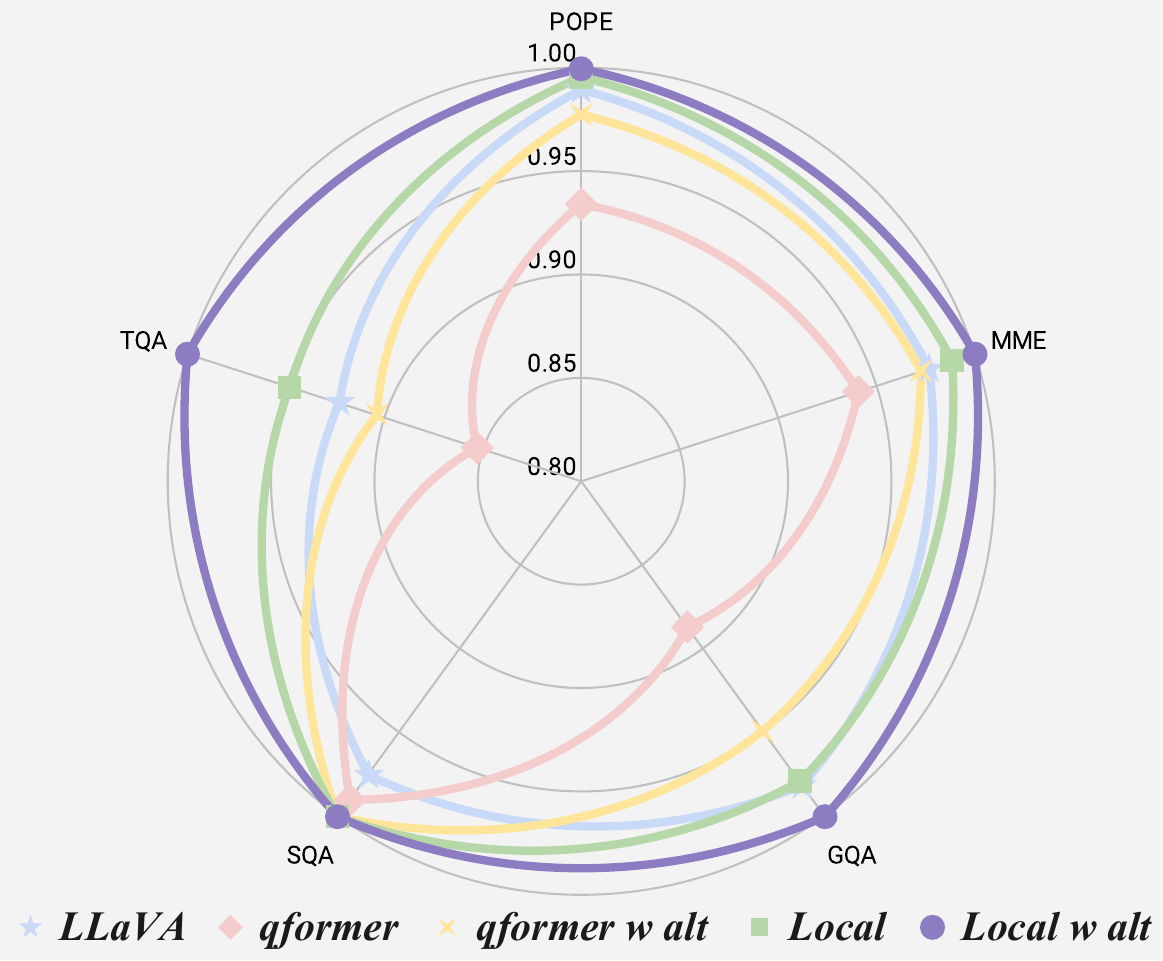}
\caption{\textbf{Significance of Alternating Training.} The reported values represent the performance ratio of baselines to the best one: Local with alternating training.}
    \label{fig:teaser_alt}
  \end{center}
\vspace{-0.3cm}
\end{wrapfigure}

\textbf{Local Feature Compression.} In our approach to local feature compression, we implement a query former architecture leveraging $N_q$ query embeddings, denoted as $M_Q\in\mathbb{R}^{N_q\times D_I}$. Here, $D_I$ represents the dimensionality of the image features obtained from the vision encoder. Notably, we strategically set $N_q$ to be smaller than the original token count derived from the vision encoder. This deliberate reduction serves to alleviate the computational burden while preserving essential information. Through the application of an attention mechanism, we orchestrate the interaction between these query embeddings and the local image features. Here, the learnable embeddings act as the query for attention, directing the model's focus towards pertinent aspects of the local features. The resultant outcome is a condensed representation of local features, consisting of $N_q$ tokens. By judiciously balancing computational efficiency with information retention, our compression strategy enhances the scalability and effectiveness of LMMs in handling diverse tasks.

\textbf{Text-Guided Router.} Our approach seeks to further alleviate computational burden by feature selection. We argue that not all parts of the local image are relevant to the questions posed. For instance, in Fig.~\ref{fig:router1}, the question \textit{"What breed is the dog?"} pertains only to specific local image regions, indicating that discarding irrelevant features can significantly reduce abundant image information. In this work, we explore a simple cosine-similarity routing strategy for its simplicity and effectiveness. Given the text embedding $z_x\in\mathbb{R}^{L_x\times D}$ and the projected local image feature $z_v\in\mathbb{R}^{L_v\times D}$, we compute scores as $S\in\mathbb{R}^{L_v\times L_x}=z_vz_x^T$. Averaging text tokens and applying softmax to image tokens yields $S_{cosine}\in\mathbb{R}^{L_v}$. Once scores or relevance indicators are obtained for each local feature, we employ an adaptive selection strategy. Specifically, we sort scores from highest to lowest and select features until the accumulated score surpasses a threshold $\gamma$. This hyperparameter balances the efficiency and completeness of local features. Our experiments reveal that selecting specific local features does not diminish performance. On the contrary, by disregarding irrelevant features and using fewer tokens, we achieve superior performance across most benchmarks. During training, Gaussian noise from $\mathcal{N}(0,0.1)$ is added to the selection score to maintain the diversity of representations.

\textbf{Alternating Training Scheme.} Our training methodology for the vision-language adapter and local compression layer involves a nuanced three-stage process. Initially, in Stage I (see Fig.~\ref{fig:stage1}), the adapter undergoes training using the global image. Subsequently, in Stage II (see Fig.~\ref{fig:stage2}), the adapter remains fixed while the local compression layer is exclusively trained using local patches. Finally, in Stage III (see Fig.~\ref{fig:router1}), both global and local features are simultaneously trained. While we delve into theoretical underpinnings in Section \ref{sec:theory}, empirical insights also bear significance. Our experimentation reveals that \textbf{simultaneously training the adapter and local compression layer in a single stage yields suboptimal performance}. This discrepancy arises from the model's predominant focus on global features, as the global feature requires only projection with no information loss, making it easier to learn. Hence, we confine the use of local patches to Stage II for compression layer training. This approach ensures sequential learning, first projection, then compression of local features (Local vs Local w. alt in Fig.~\ref{fig:teaser_alt}). Additionally, alternating training can bridge the performance gap between two common adapters: MLP and query former, as mentioned earlier. Specifically, when employing attention-based models as adapters, which offer more flexibility but may exhibit inferior performance compared to simple MLP adapters~\citep{zeng2023matters}, we find that alternating training significantly enhances performance (query former vs query former w. alt\footnote{Firstly, the model learns query embeddings to interact with image features, and secondly, it projects the image features to the LLM dimension. Similarly, in this alternating training approach, we first learn the projection head and then focus on learning attention mechanism parameters and query embeddings.} in Fig.~\ref{fig:teaser_alt}). Such a scheme may illuminate future work, facilitating the training of more complex yet flexible adapter options.

\begin{figure*}[t]
\centering
\subfigure[\textbf{Adaptive Slicing.}]{
\begin{minipage}[t]{0.35\linewidth}
 \includegraphics[width=\linewidth]{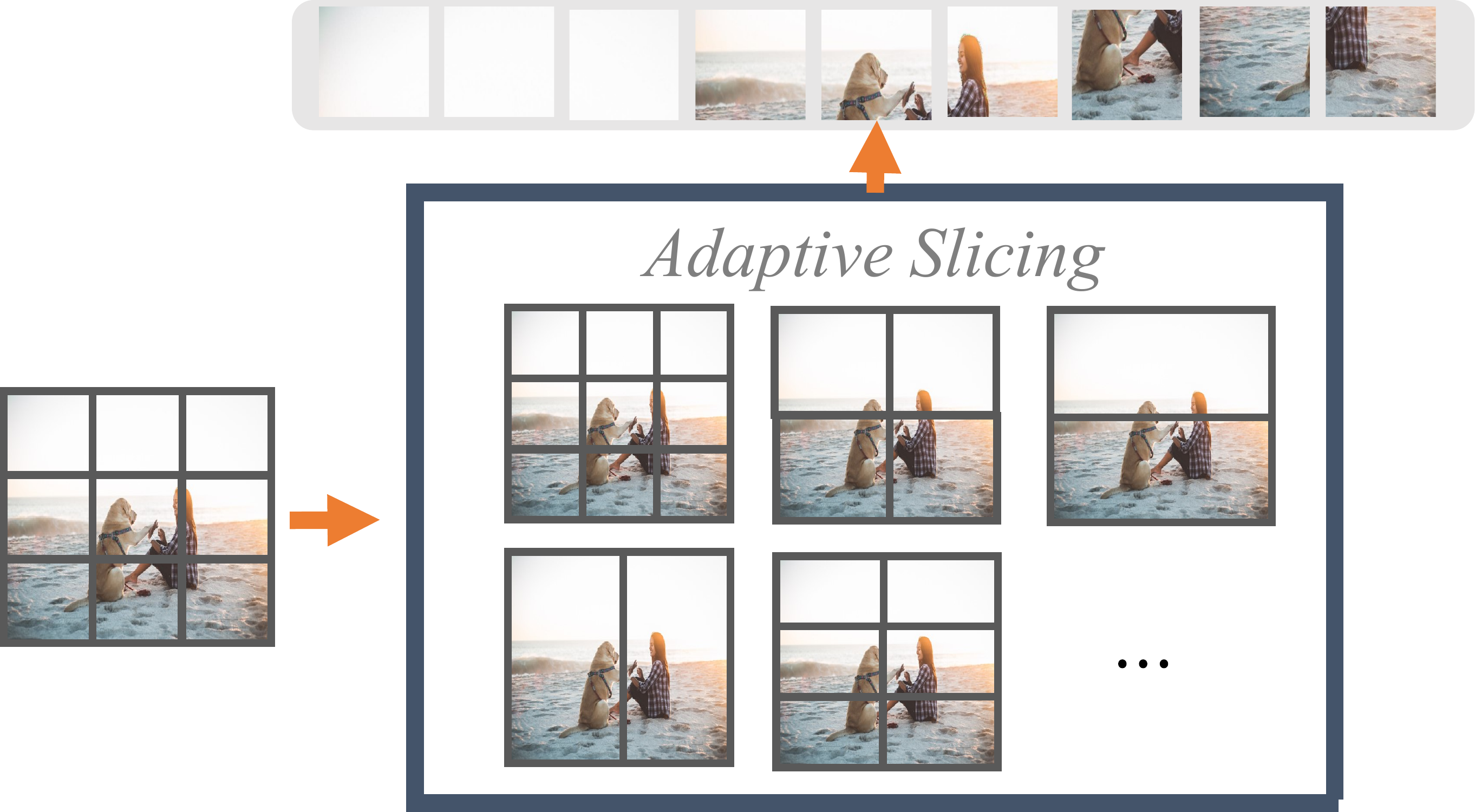}
    \label{fig:slcing}
\end{minipage}%
}%
\subfigure[\textbf{Stage I.}]{
\begin{minipage}[t]{0.45\linewidth}
\centering
 \includegraphics[width=\linewidth]{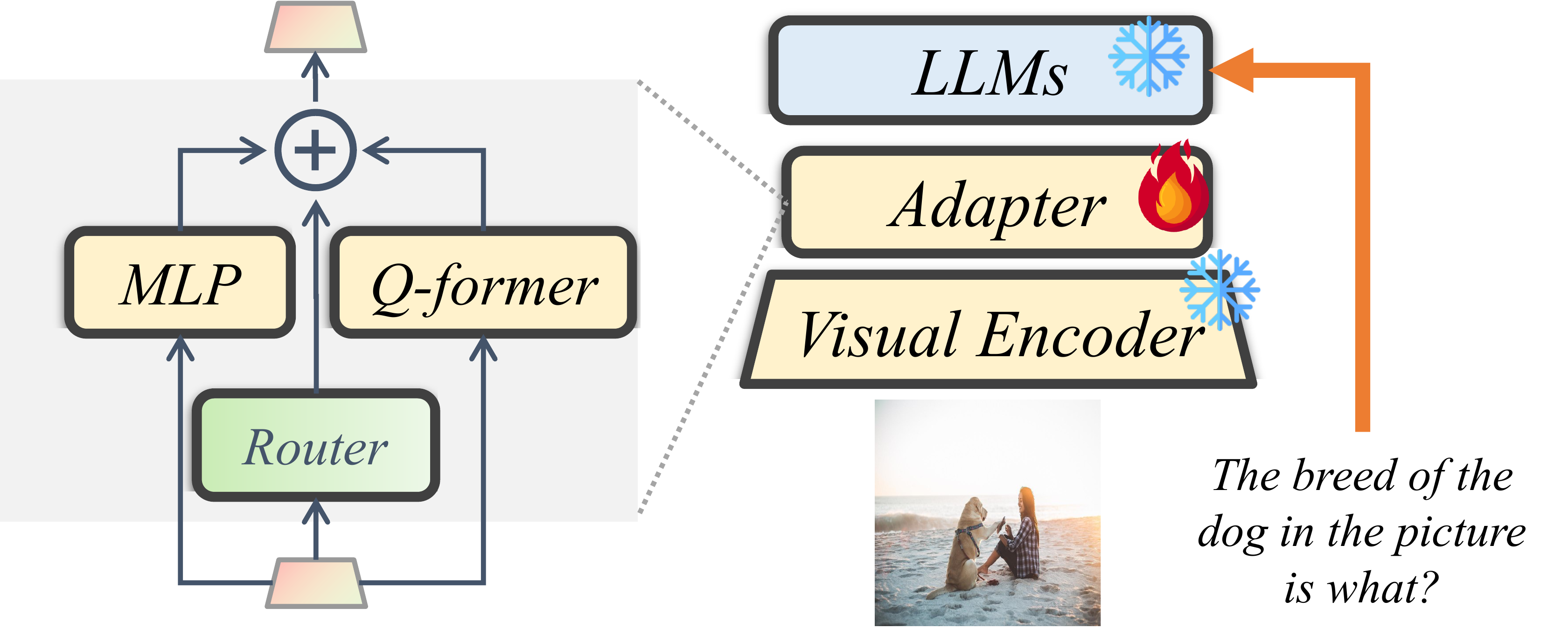}
    \label{fig:stage1}
\end{minipage}%
}%

\subfigure[\textbf{Stage II.}]{
\begin{minipage}[t]{0.38\linewidth}
 \includegraphics[width=\linewidth]{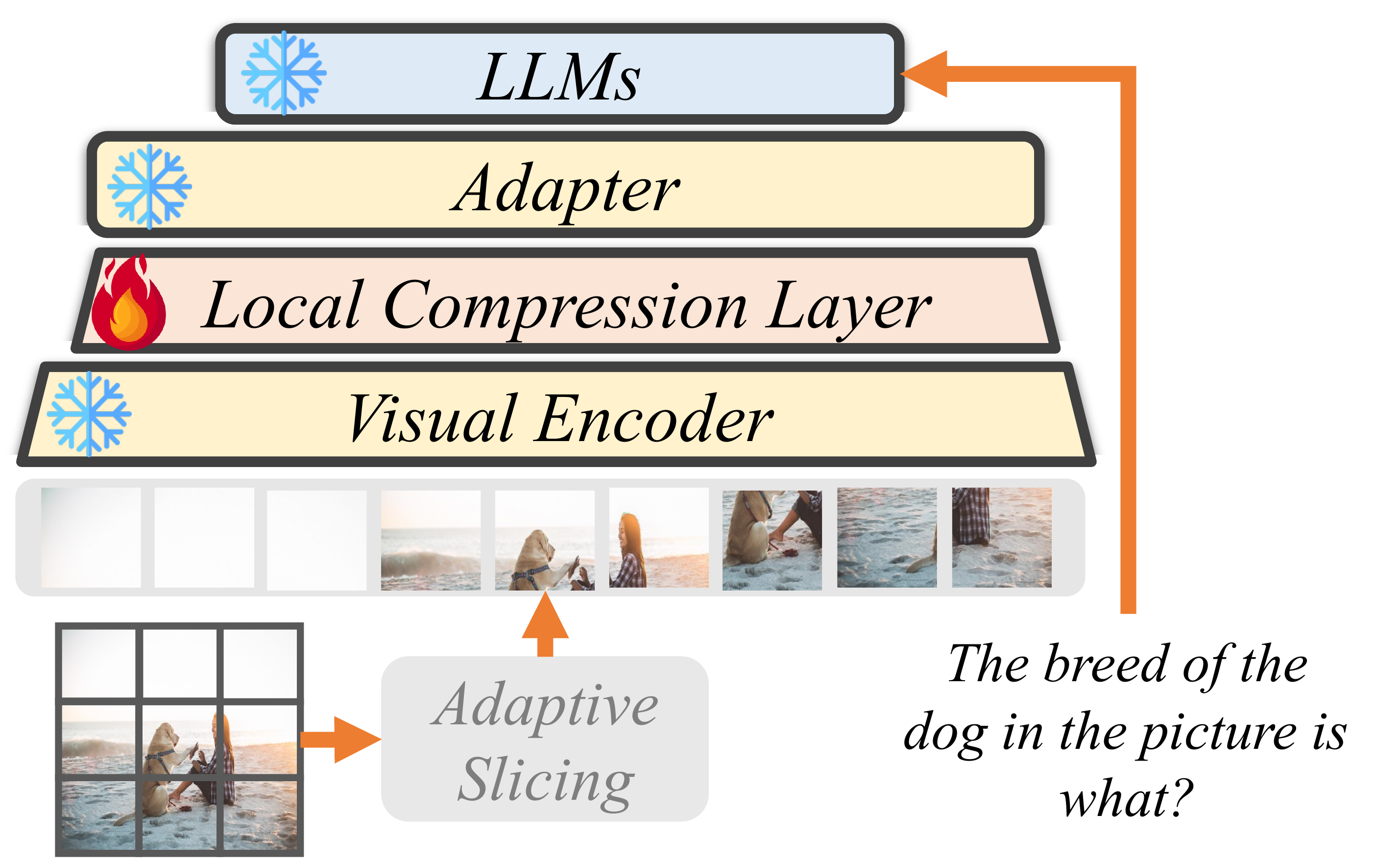}
    \label{fig:stage2}
\end{minipage}%
}%
\subfigure[\textbf{Stage III.}]{
\begin{minipage}[t]{0.55\linewidth}
\centering
 \includegraphics[width=\linewidth]{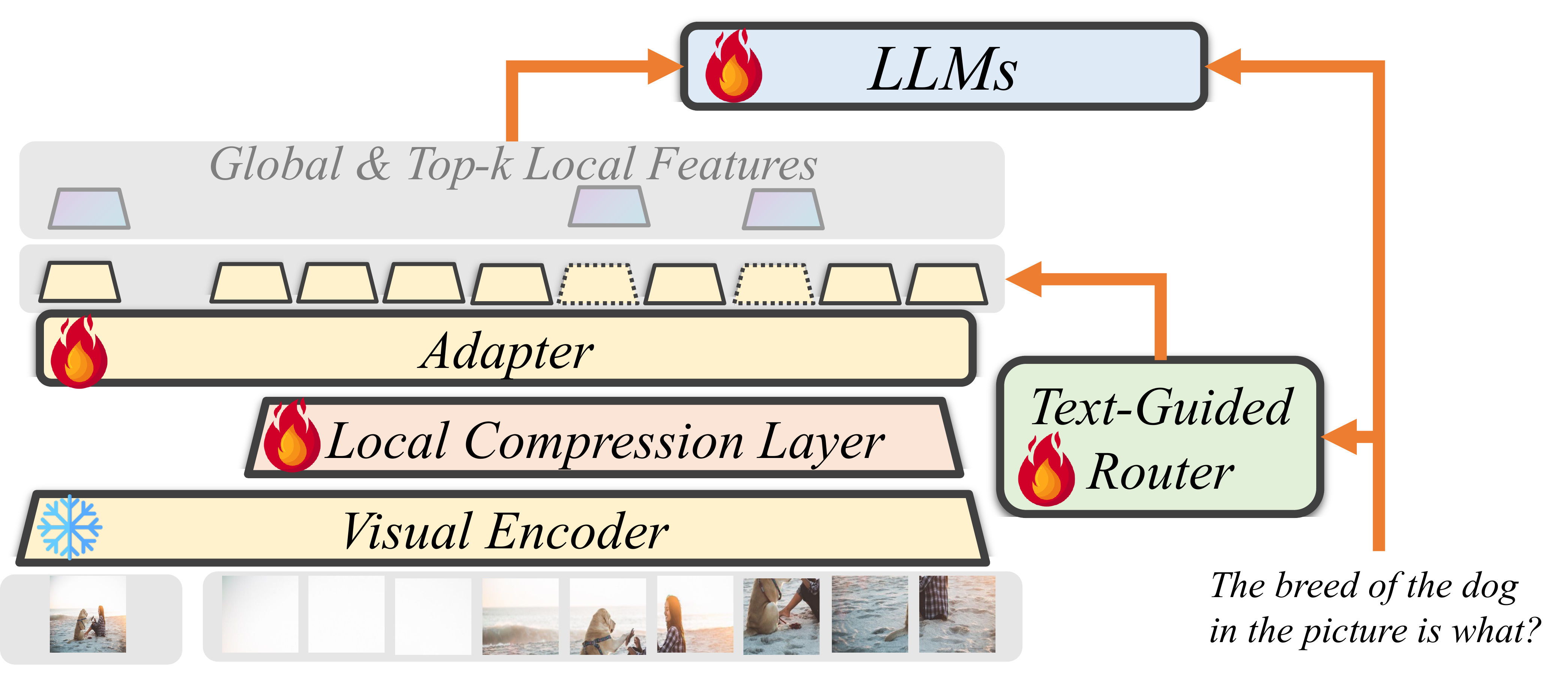}
    \label{fig:router1}
\end{minipage}%
}%
\centering
\vspace{-0.2cm}
\caption{\textbf{Innovative Training Approach}: (b) Refining the visual adapter with mixture of experts, (c) optimizing local compression layers, and (d) instructions fine-tuning. Here, \abbr efficiently processes images via slicing, projection, and selecting pertinent local features aligned with query.}
\label{fig:router}
\end{figure*}
\vspace{-0.1cm}
\subsection{Importance of Alternating Training for Optimizing Bilinear Functions}\label{sec:theory}
\vspace{-0.1cm}
\begin{tcolorbox}[top=1pt, bottom=1pt, left=1pt, right=1pt]
\textbf{Takeaway:}~\textit{Alternating Training is pivotal for the success of \abbr. Our demonstration in this subsection will also shed light on why it's common practice to initially freeze one modality in multi-modal learning and optimize the adapter of one modality before engaging in joint optimization across multiple modalities. All the proof can be found in~\cref{sec:app_proof}.}
\end{tcolorbox}
Bilinear forms are prevalent in deep learning models, especially in multi-modal learning where the representation from two different modalities is frequently aligned through a dot product. Let the target matrix $X \in \mathbb{R}^{d\times d}$ be expressed as $X = ab^{\top} + ba^{\top}$, where $a, b \in \mathbb{R}^d$ are two normalized vectors. Our objective is to find the rank-$1$ matrix to approximate $X$, which leads to the following optimization problem:
\[
\min\limits_{u, v \in \R^d} \; \L(u,v) = \frac{1}{2}\left|uv^{\top} - X\right|^2
\]
In LMMs, the vision encoder and adapter can be perceived as the vision modality, while others are categorized as the text modality, and the target $X$ can be seen as the best LMM. Within our framework, we treat the adapter and local compression layer as distinct functions, aiming to approximate the optimal modality adaptation parameter. We recognize that assuming both $a$ and $b$ to be merely vectors is a simplification that may not fully capture the complexity of the entire model. However, this simplification allows us to analyze the problem more effectively and derive valuable insights from it.
 It is well known that the optimal solution for $u$ is aligned with the top eigenvector of $XX^{\top}$, i.e.
\[
XX^{\top} = (ab^{\top} + b^{\top}a)(ba^{\top} + a^{\top}b) = A\underbrace{\left(\begin{array}{cc}
1 & a^{\top}b \\
a^{\top}b & 1
\end{array}\right)}_{:=M} A^{\top}
\]
where $A = (a, b) \in \R^{d\times 2}$ and $M \in \R^{2\times 2}$. Since $X$ is constructed through $a$ and $b$, $u$ has to lie in the subspace spanned by $a$ and $b$ and thus can be written $u$ as
\[
u = \alpha a + \beta b = A\underbrace{\left(
\begin{array}{c}
\alpha \\
\beta
\end{array}
\right)}_{:= z}
\]
Hence, the optimal solution for $z = (\alpha, \beta)^{\top}$ should be aligned with the largest eigenvector of matrix $M$. 
Let $u_0$ and $v_0$ be the initial solution and is given in the following form
\[
u_0 = \alpha_0 a + \beta_0 b, \; v_0 = \beta_0 a + \alpha_0 b
\]
where $\alpha_0, \beta_0 \in \R$ are two scales. Here, we utilize the fact that $u$ and $v$ have to lie in the subspace spanned by $a$ and $b$. Then we state the following theorem:

\begin{thm}
Using the gradient descent method, we update the solution $u_t$ and $v_t$ as
\begin{eqnarray}
u_{t+1} = u_t - \eta\left(u_t v_t^{\top} - X\right) v_t, \quad v_{t+1} = v_t - \eta\left(v_t u_t^{\top} - X\right)u_t
\label{equ:train-sim}
\end{eqnarray}
Simultaneously updating $u$ and $v$ using Eq.~(\ref{equ:train-sim}) is less ideal for optimizing the objective function of bilinear form, as the gradient descent update does not necessarily converge to the optimal solution.
\label{them:sim}
\end{thm}

We will demonstrate that the issue with gradient descent (or more accurately, simultaneously updating $u$ and $v$) can be effectively addressed by alternating optimization. Specifically, we will optimize $v$ with fixed $u$, and then optimize $u$ with fixed $v$. We will show that this approach converges to the optimal solution by alternating optimization. 
\begin{thm}
Let $u_0 = \alpha_0 a + \beta_0 b$. We rewrite the sequential solution $u_t$ obtained by alternating optimization as $u_t = \alpha_t a + \beta_t b$. $z_t = (\alpha_t, \beta_t)^{\top}$ evolves over iterations by $z_{t+1} = \frac{1}{|u_t|^2|v_t|^2}M^2 z_t$
\label{them:alter}
\end{thm}
That is, alternating optimization ensures that $z_{t} \propto M^{2t}z_0$, implying that $z_t$ is guaranteed to converge to the largest eigenvector of $M$, thus resolving the limitation of gradient descent.

\subsection{Expanding Dataset Scope with Challenging Reasoning Tasks}
\begin{figure}
\vspace{-0.1cm}
    \centering
    \includegraphics[width=0.95\linewidth]{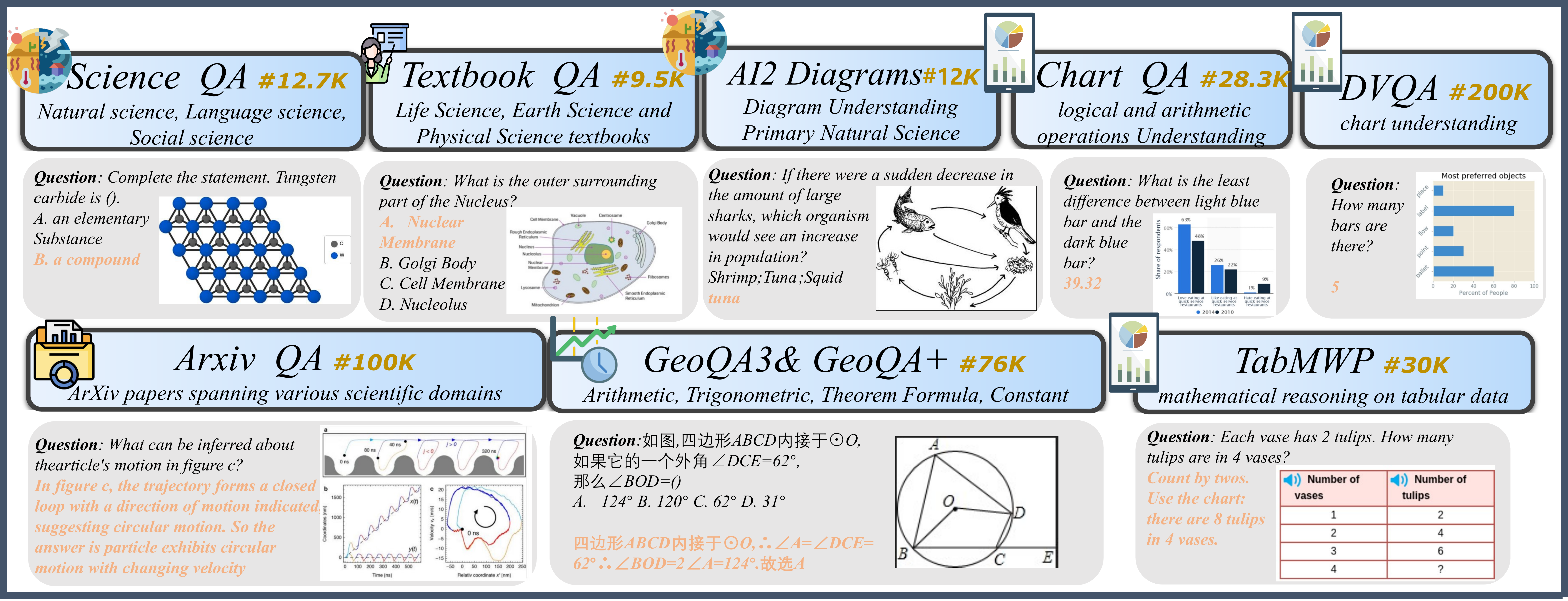}
    \caption{\textbf{The used science and mathematical reasoning tasks in this paper}.}
    \label{fig:science_dataset}
    \vspace{-0.1cm}
\end{figure}

\textbf{Generation of Source Data and Instruction Data.} The creation of SMR involves a meticulous amalgamation of publicly available datasets, comprising Arxiv-QA~\citep{li2024multimodal}, ScienceQA~\citep{lu2022learn}, MATH-Vision~\citep{wang2024measuring}, TextBookQA~\citep{kembhavi2017you}, GeoQA3~\citep{chen2021geoqa}, Geometry3K~\citep{lu2021inter}, TabMWP~\citep{lu2022dynamic}, DVQA~\citep{kafle2018dvqa}, AI2D~\citep{kembhavi2016diagram}, and ChartVQA~\citep{masry2022chartqa}. The variety of question types and associated images sourced from these datasets is depicted in Fig.~\ref{fig:science_dataset}, presenting a distinctive challenge to existing instruction datasets, as illustrated in Fig.~\ref{fig:vqa_dataset}. The disparities between SMR and conventional instruction tuning datasets manifest in two key aspects: (1) \textbf{Challenging Reasoning Tasks}. Many of the tasks in Physical/Social science and mathematics demand advanced reasoning abilities. Notably, datasets such as Arxiv-QA, GeoQA3, and TabMWP offer complete reasoning paths, including intermediate steps for deriving final results. In such cases, the model is tasked not only with mastering foundational knowledge but also with articulating complex reasoning processes—a notably more demanding endeavor. (2) \textbf{Demand for Image Detail Understanding}. All tasks necessitate a profound understanding of visual details because many images contain rich annotation information or questions requiring comprehensive visual analysis. This aspect is particularly beneficial for training our high-resolution framework. Further elucidation on datasets and specific construction methodologies can be found in~\cref{sec:related_work}. To ensure the accuracy of our data, we carefully filter it after collection. This involves identifying and fixing issues like blurry images or jumbled text, unrelated image-text pairs, and incorrect reasoning paths that can't lead to correct answers or might lead to wrong conclusions. For the latter, we use GPT-4V to create new, accurate reasoning paths. 

\textbf{Statistics and Analysis.} In Fig.~\ref{fig:lengthes}, we illustrate the differences in statistics between SMR and existing instruction tuning datasets. To standardize multi-round conversations, we aggregate them into one-round and calculate the average length. We employ LLaVA to determine the maximum length such that $99\%$ of the data falls within the interval. LLaVA~\citep{liu2023improved} comprises 665K instruction tuning data instances characterized by short queries and answers. Similarly, LLaVAR~\citep{zhang2023llavar} exhibits comparable patterns to LLaVA. Conversely, ShareGPT4V~\citep{chen2023sharegpt4v}, a more comprehensive dataset derived from 100K high-quality captions generated by advanced GPT4-Vision models, features longer generation lengths, indicative of more detailed and complex captions. In contrast, SMR demonstrates longer query texts compared to existing training corpora, reflecting the need for detailed descriptions or background information to elucidate scientific or mathematical problems. Additionally, since some of our datasets focus solely on question-answer tasks, approximately $50\%$ of instances feature shorter answer lengths. However, owing to the task complexity, particularly those datasets with extensive reasoning paths aiming to train models to comprehend intricate chains of reasoning, SMR exhibits a higher ratio of longer sentences compared to LLaVA and LLaVAR.

\begin{figure*}[t]
\subfigure[{Average length of query text.}]{
\begin{minipage}[t]{0.49\linewidth}
\centering
 \includegraphics[width=\linewidth]{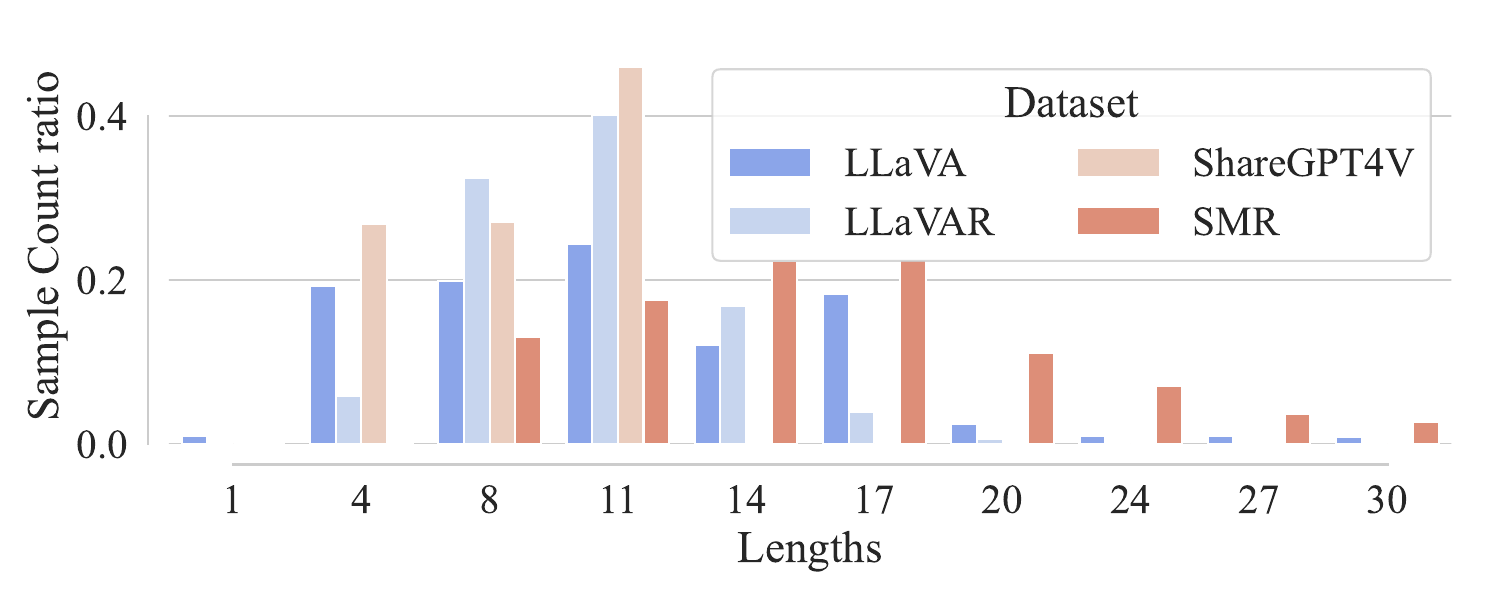}
    \label{fig:length_query}
\end{minipage}%
}%
\subfigure[{Average length of answer text.}]{
\begin{minipage}[t]{0.49\linewidth}
 \includegraphics[width=\linewidth]{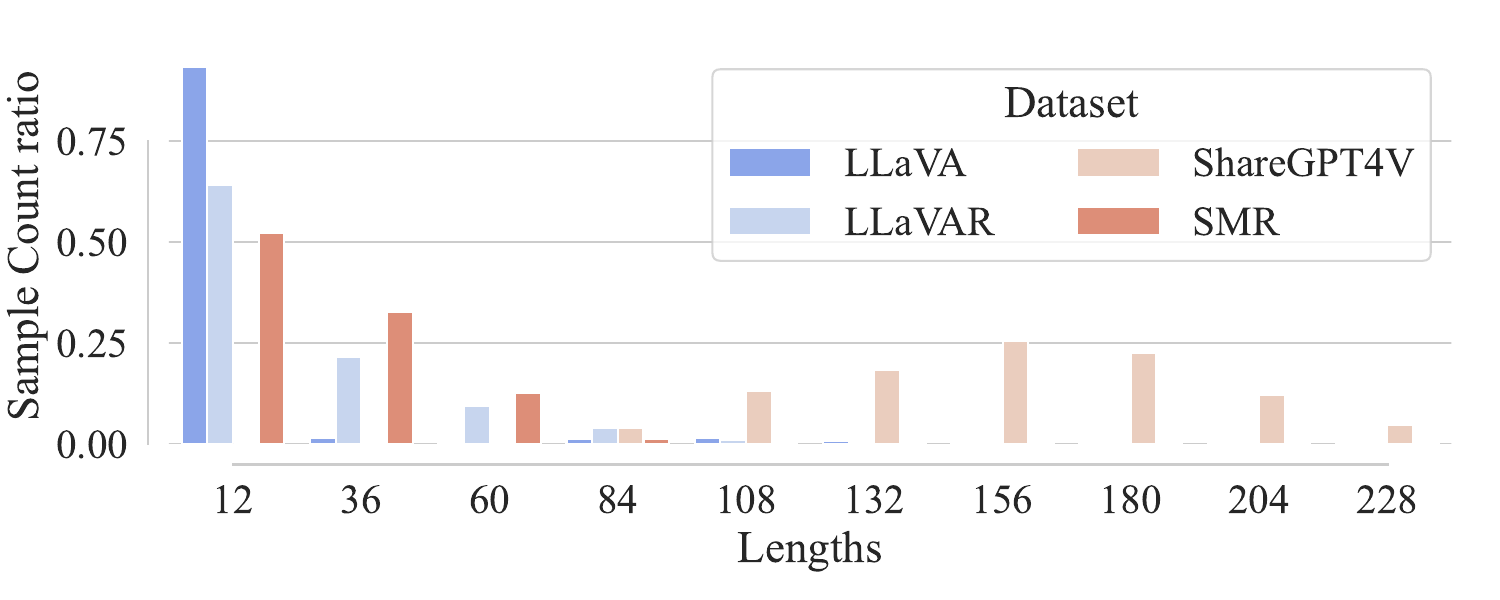}
    \label{fig:length_answer}
\end{minipage}%
}%
\centering
\vspace{-0.2cm}
\caption{\textbf{Comparison of the average lengths of query and answer texts across different datasets.}}
\label{fig:lengthes}
\end{figure*}

\vspace{-0.1cm}
\section{Experiment}
\vspace{-0.1cm}
\abbr is evaluated against both open-source and closed-source models across various domains, encompassing General QA and Open-ended Generation, Math Reasoning, Science, and Hallucination benchmarks, totaling 15 benchmarks. We conduct a comprehensive analysis of the evaluation benchmarks, their associated metrics, and the training hyperparameters for both the initial vision-language alignment pretraining and the subsequent visual instruction tuning stages. Details are provided in Appendix~\ref{sec:exp_app}. The training process utilizes 8×A100 (80G) GPUs.

\vspace{-0.2cm}
\subsection{Numerical Results}
\vspace{-0.2cm}
\textbf{General QA and Open-ended Generation.} We assess the performance of \abbr across a range of benchmarks, encompassing both academic-task-oriented assessments and recent benchmarks tailored for instruction-following LLMs, totaling 9 in all. Our results demonstrate that \abbr consistently achieves superior performance across all benchmarks, even when compared to LLMs of similar scale, despite utilizing significantly smaller pretraining and instruction tuning datasets than other methods~\citep{li2023monkey,bai2023qwen,li2023otterhd}. Notably, \abbr-8B even surpasses Gemini Pro on several benchmarks such as MMB and MME. Prior research has frequently indicated that the LoRA model performs comparably to full fine-tuning, a trend that holds true across many of our benchmarks. However, we observe that models trained with LoRA struggle in instruction-following tasks. This observation is bolstered by the performance gap observed in the LLaVA-bench between \abbr-8B and \abbr-8B$^\dagger$. Furthermore, in our evaluation of MathVerse, we find that while the model prompt explicitly requested concise answers, \abbr-8B$^\dagger$ consistently generates responses with intermediate reasoning, a behavior absent in \abbr-8B. We delve into a detailed analysis of these inconsistencies in~\cref{sec:extend_exp}.

\begin{table}[]
\caption{\textbf{Comparison with SoTA methods on academic-task-oriented datasets and benchmarks for instruction-following LMMs.} VQA$^T$: TextVQA~\citep{8953586}, MMB: MMBench~\citep{liu2024mmbench}, MMB$^C$: MMBench-Chinese~\citep{liu2024mmbench}; MMMU$_v$: validation set of MMMU~\citep{yue2023mmmu}; LLaVA$^W$: LLaVA-Bench (In-the-Wild)~\citep{liu2023visual}; MME$^{P,C}$: Perception and Cognition in MME~\citep{fu2023mme}. The best result is marked in bold, and the second best result is underlined. $^\dagger$ means using LoRA during the instruction tuning phase.}
\resizebox{\textwidth}{!}{%
\begin{tabular}{llcccccccccc}
\toprule
\Gray
\textbf{Method} & \textbf{LLM} & \textbf{VQA$^T$} & \textbf{GQA} & \textbf{VQA$^{v_2}$} & \textbf{MMB} & \textbf{MMB$^C$} & \textbf{MMMU$_v$} & \textbf{LLaVA$^W$} & \textbf{MME$^P$} & \textbf{MME$^C$} & MM-Vet \\
InstructBLIP~\citep{dai2024instructblip} & \textit{Vicuna-7B} & 50.10 & 49.20 & - & 36.00 & 23.70 & 32.90 & 60.90 & - & - & 26.20 \\
Qwen-VL~\citep{bai2023qwen} & \textit{Qwen-7B} & 63.80 & 59.30 & 78.80 & 38.20 & 7.40 & 35.90 & - & - & -  & - \\
LLaVA-1.5~\citep{liu2023improved} & \textit{Vicuna-7B} & 58.20 & 62.00 & 78.50 & 64.30 & 58.30 &  - & 65.40 & 1510 & - &30.50 \\
LLaVA-1.5~\citep{liu2023improved} & \textit{Vicuna-13B} & 61.30 & 63.30 & 80.00 & 67.70 & 63.60 & 36.40 & 72.50 & 1531 & 295& 35.40\\
ShareGPT4V~\citep{chen2023sharegpt4v} & \textit{Vicuna-7B} & - & - & 80.60 & 68.80 & 62.20 & - & 72.60 & 1567 & 303 &37.60 \\
LLaVA-1.5~\citep{liu2023improved} & \textit{Llama3-8B} & 58.94 & 61.94 & 79.49 & 72.94 & 67.70 & 38.00 & 70.50 & 1544 & 328 & 34.80 \\ \cmidrule{2-12}  \Gray
\multicolumn{12}{c}{\textit{With High Resolution}} \\
OtterHD-8B~\citep{li2023otterhd} & \textit{Fuyu-8B} & - & - & - & 58.30 & - & - & - & 1223 & 331 & 26.30 \\
Monkey~\citep{li2023monkey} & \textit{Qwen-7B} & - & 60.70 & 80.30 & 72.40 & 67.50 & - & - & 1522 & \textbf{401} & 33.00 \\
LLaVA-HD~\citep{liu2023improved} & \textit{Vicuna-13B} & 62.50 & \textbf{64.70} &\textbf{81.80} & 68.80 & 61.90 & - & 72.00 & 1500 & - & - \\ 
LLaVA-NeXT~\citep{liu2024llavanext} & \textit{Vicuna-13B} & - & - & - & - & - & 35.90 & 72.30 & {1575} & 316 & \textbf{48.40} \\ \cmidrule{2-12} 
\Gray  
\multicolumn{12}{c}{\textit{Ours}} \\
\texttt{SliME-7B} & \textit{Vicuna-7B} & 64.39 & 63.13 & 80.32 & 69.32 & 61.85 & 37.20 & \textbf{76.10} & 1544 & \underline{383} & 35.40 \\
\texttt{SliME-8B} & \textit{Llama3-8B} & {64.76} & \underline{63.94} & 80.69 & \underline{75.00} & \textbf{71.80} & \textbf{41.20} & \underline{73.90} & \underline{1578} & 337 & 37.40 \\
\texttt{SliME-8B$^\dagger$} & \textit{Llama3-8B} & \underline{65.26} & \underline{63.94} & \underline{80.79} & \textbf{75.42} & \underline{70.96} & \underline{40.80} & 64.90 & {1573} & 346 & 36.80 \\
\texttt{SliME-13B} & \textit{Vicuna-13B} & \textbf{66.11} & 63.60 & 80.43 & 71.13 & 65.20 & 38.00 & 73.10 & \textbf{1606} & 293 & \underline{41.20} \\
\cmidrule{2-12}  
\Gray
\multicolumn{12}{c}{\textit{Private models}} \\
\lightgrey{Gemini Pro}~\citep{team2023gemini} & - & \lightgrey{74.60} & - & - & \lightgrey{73.60} & \lightgrey{74.30} & \lightgrey{47.90} & - & \lightgrey{1496} & \lightgrey{436} & 59.20  \\
\lightgrey{Qwen-VL-Plus}~\citep{bai2023qwen} & - & - & - & - & - & \lightgrey{68.00} & \lightgrey{45.20} & - & \lightgrey{1681} & \lightgrey{502} & -\\
\lightgrey{GPT-4V}~\citep{gpt4} & - & \lightgrey{78.00} & - & - & \lightgrey{77.00} & \lightgrey{74.40} & \lightgrey{58.10} & - & \lightgrey{1409} & \lightgrey{517}  & 56.80\\ \bottomrule
\end{tabular}%
}
\end{table}

\textbf{Scientific, Mathematical, and Hallucination.} We further assess the hallucinatory property and mathematical proficiency of \abbr. As demonstrated in Table.~\ref{tab:math_hall}, \abbr achieves state-of-the-art performance, comparable to Gemini Pro, across all three mathematical benchmarks. Moreover, its performance on the ScienceQA-Img split and hallucination benchmarks is particularly noteworthy, affirming the efficacy of the proposed \abbr.

\begin{table}[]
\caption{\textbf{Comparison with SoTA methods on Science and Mathematical Reasoning benchmarks.} SQA$^I$: ScienceQA-IMG;  MME$^H$: the sum of scores in existence, count, position, color splits of MME benchmark. $^\dagger$ means using LoRA during the instruction tuning phase}\label{tab:math_hall}
\centering
\resizebox{\textwidth}{!}{%
\begin{tabular}{lccccc|ccc}
\toprule
\Gray
\multicolumn{6}{c}{\textit{Science and Mathematical Reasoning}} & \multicolumn{3}{c}{\textit{Hallucation}} \\ \Gray
\textbf{Method} & \textbf{LLM} & \textbf{MathVerse} & \textbf{MathVista} & \textbf{MathVision} & \textbf{ScienceQA} & \textbf{POPE} & \textbf{AMBER} & \textbf{MME$^H$} \\
InstructBLIP~\citep{dai2024instructblip} & \textit{Vicuna-7B} & - & 25.30 & - & 60.50 & - & 81.70 & - \\
Qwen-VL~\citep{bai2023qwen} & Qwen-7B & - & 33.80 & 10.53 & 67.10 & - & 84.90 & 606 \\
LLaVA-1.5 & \textit{Vicuna-7B} & - & 26.10 & 10.20 & 66.80 & 85.90 & 74.70 & - \\
LLaVA-1.5 & \textit{Vicuna-13B} & 7.60 & 26.10 & 13.10 & 71.60 & 85.90 & - & 643 \\
ShareGPT4V & \textit{Vicuna-7B} & 13.10 & 25.80 & 12.50 & 68.40 & - & - & - \\
LLaVA-1.5 & LLama3-8B & 13.80 & 28.40 & 14.75 & 77.64 & 85.10 & 85.00 & 634 \\
\Gray
\multicolumn{9}{c}{\textit{With High Resolution}} \\
OtterHD-8B~\citep{li2023otterhd} & Fuyu-8B & - & 23.40 & - & - & 86.00 & {89.10} & - \\
Monkey~\citep{li2023monkey} & Qwen-7B & - & 34.80 & - & 69.40 & - & 86.00 & - \\
LLaVA-HD~\citep{liu2023improved} & \textit{Vicuna-13B} & - & - & - & 71.00 & \textbf{86.30} & - & - \\
\Gray
\multicolumn{9}{c}{\textit{Ours}} \\
\texttt{SliME-7B} & \textit{Vicuna-7B} & 17.50 & 37.50 & 16.12 & 76.80 & 85.40 & 87.80 & 633 \\
\texttt{SliME-8B} & \textit{LLama3-8B} & \textbf{22.90} & \underline{43.30} & \underline{16.78} & \textbf{84.18} & \underline{86.00} & 88.90 & \underline{671} \\
\texttt{SliME-8B$^\dagger$} & \textit{LLama3-8B} & \underline{21.80} & \textbf{43.60} & {16.12} & \underline{84.13} & \underline{86.00} & \textbf{90.10} &{645} \\
\texttt{SliME-13B} & \textit{Vicuna-13B} & 19.00 & 40.80 & \textbf{18.09} & 80.17 & \textbf{86.30} & \underline{89.40} & \textbf{673} \\
\Gray
\multicolumn{9}{c}{\textit{Private models}} \\
\lightgrey{Gemini Pro}~\citep{team2023gemini} & - & \lightgrey{22.30} & \lightgrey{45.20} & \lightgrey{17.11} & - & - & - & \lightgrey{560} \\
\lightgrey{Qwen-VL-Plus}~\citep{bai2023qwen} & - & \lightgrey{11.80} & \lightgrey{43.30} & - & - & - & - & \lightgrey{670} \\
\lightgrey{GPT-4V}~\citep{gpt4} & - &\lightgrey{ 38.30} & \lightgrey{49.90} & \lightgrey{22.37} & - & - & \lightgrey{87.40} & \lightgrey{595} \\ \bottomrule
\end{tabular}%
}
\end{table}

\vspace{-0.1cm}
\subsection{Ablation Studies and Analysis}

\begin{table}[tb]
    \centering
    \begin{minipage}{.45\textwidth}
      \centering
      \caption{\small \textbf{Ablation results} on the global MOE \colorbox[HTML]{D9EAD3}{ },the token number of the compression layer \colorbox[HTML]{CFE2F3}{ }, the router parameter $\gamma$ \colorbox[HTML]{EAD1DC}{ } and the training data \colorbox[HTML]{FFF2CC}{ }.}\label{tab:ablation}
      \resizebox{0.9\linewidth}{!}{%
      \begin{tabular}{@{}lccccc@{}}
\toprule
\multicolumn{1}{c}{Dataset} & \textbf{POPE} & \textbf{GQA} & \textbf{SQA} & \textbf{VQA$^T$} & \textbf{AMBER} \\ \midrule
\texttt{Baseline}  & 83.80 & 61.94 & 74.30 & 58.94  & 87.70 \\ \cmidrule{2-6}
\rowcolor[HTML]{D9EAD3} 
\texttt{Global MOE}  & 84.10 & 62.58 & 77.00 & 60.57  & 88.70 \\ \cmidrule{2-6}
\rowcolor[HTML]{CFE2F3} 
\texttt{$N_q=64$ } & 83.80 & 62.91 & 76.15 & 63.09 & 87.60 \\
\rowcolor[HTML]{CFE2F3} 
\texttt{$N_q=196$}  & 84.00 & 62.51 & 75.56 & 62.48  & 87.40 \\
\rowcolor[HTML]{CFE2F3} 
\texttt{$N_q=144$}  & 84.20 & 62.96 & 77.09 & 63.58  & 88.20 \\ \cmidrule{2-6}
\rowcolor[HTML]{EAD1DC} 
{\texttt{$\gamma=90\%$}}  & 84.20 & 62.74 & 76.05 & 64.05  & 88.10 \\
\rowcolor[HTML]{EAD1DC} 
{\texttt{$\gamma=75\%$}}  & 84.50 & 63.09 & 77.44 & 63.83 & 88.40 \\
\rowcolor[HTML]{EAD1DC} 
{\texttt{$\gamma=50\%$}}  & 84.90 & 63.11 & 77.09 & 62.75 & 88.40 \\ \cmidrule{2-6}
\rowcolor[HTML]{FFF2CC} 
\texttt{With SMR}  & 84.50 & 62.58 & 82.29 & 59.89  & 88.40 \\ \cmidrule{2-6}
\texttt{Final}   & 84.9 & 63.94 & 84.18 & 64.76 & 88.90 \\ \bottomrule
\end{tabular}%
}
    \end{minipage}%
    \hspace{0.02\textwidth}
    \begin{minipage}{.5\textwidth}
      \raggedleft
    \caption{\textbf{Ablation results} on different treatments for global and local features \colorbox[HTML]{D9EAD3}{ }, and two different training strategies \colorbox[HTML]{CFE2F3}{ }.}\label{tab:ablation_2}
      \resizebox{0.95\linewidth}{!}{%
\begin{tabular}{@{}lccccc@{}}
\toprule
\multicolumn{1}{c}{Dataset} & \textbf{POPE} & \textbf{GQA} & \textbf{SQA} & \textbf{VQA$^T$} & \textbf{AMBER} \\ \midrule
\texttt{Baseline} & 84.20 & 61.94 & 74.30 & 58.94 & 87.70 \\  \cmidrule{2-6}
\rowcolor[HTML]{D9EAD3} 
\texttt{LLaVA-HD} & 85.00 & 62.48 & 72.97 & 61.48 & 88.20 \\ \rowcolor[HTML]{D9EAD3} 
\texttt{Monkey} & 83.10 & 60.70 & 73.85 & 62.14 & 86.90 \\ \cmidrule{2-6} 
\rowcolor[HTML]{CFE2F3} 
\texttt{E2E} & 82.90 & 61.90 & 74.12 & 59.69 & 86.00 \\
\rowcolor[HTML]{CFE2F3} 
\texttt{$\;\;$only Global} & 82.10 & 62.43 & 75.26 & 43.65 & 84.30 \\
\rowcolor[HTML]{CFE2F3} 
\texttt{$\;\;$only Local} & 43.40 & 38.70 & 71.24 & 58.22 & 78.20 \\ \cmidrule{2-6} \rowcolor[HTML]{CFE2F3} 
\texttt{Alternating} & 84.20 & 62.96 & 77.44 & 63.58 & 88.40 \\
\rowcolor[HTML]{CFE2F3} 
\texttt{$\;\;$only Global} & 84.10 & 62.78 & 76.01 & 59.42 & 87.30 \\
\rowcolor[HTML]{CFE2F3} 
\texttt{$\;\;$only Local} & 82.10 & 53.21 & 76.23 & 62.45 & 85.60 \\ \bottomrule
\end{tabular}%
}
    \end{minipage} 
\end{table}

\textbf{Why Different Strategies for Global and Local Treatment are Necessary?} When comparing two strategies from LLaVA-HD~\citep{liu2023improved} and Monkey~\citep{li2023monkey} with identical hyperparameters and slicing strategies, it becomes apparent why different treatment strategies for global and local features are essential. LLaVA-HD does not compress local features, in contrast, all image features are directly projected by an MLP, resulting in a maximal context size of 4096. That is, this approach significantly increases both training and inference times. Conversely, Monkey compresses all global and local image tokens using $144$ learnable query embeddings, akin to \abbr. Despite LLaVA-HD introducing more image tokens, our approach outperforms it. For instance, in the SQA dataset, \texttt{only global} of \abbr achieves commendable performance, highlighting the importance of global context in SQA tasks. However, as image features are primarily dominated by local image details, LLaVA-HD is detrimental to the SQA dataset. Conversely, for datasets like VQA$^T$, which demand more image details, LLaVA-HD consistently achieves performance gains. Notably, LLaVA-HD slightly outperforms \abbr in the POPE benchmark. This is likely because POPE questions are simplistic and focus on single objects in the image. Therefore, even with some loss of local detail information, LLaVA-HD can answer such questions more effectively. Monkey's approach, compressing all features, surpasses LLaVA-HD in SQA and VQA$^T$ by nearly $1$ point. However, it performs inferiorly in other benchmarks, emphasizing the importance of maintaining the global context without compression. In contrast, \abbr maintains all the global context and provides additional image detail with compression, yielding promising results regardless of whether the datasets prioritize global context or local details.

\textbf{Impact of Alternating Training on Performance.} In this part, we investigate the effect of alternating training on model performance. To assess the significance of Alternating Training, we initially compare the performance directly (lines 4 \& 7 in Table.~\ref{tab:ablation}), revealing a substantial performance gap between them. To further explore this phenomenon, we isolate the global and local features as image tokens respectively to assess the amount of image information provided by each. Notably, for the model that is trained end-to-end, we observe that utilizing only global features yields satisfactory results, while the local features are inadequately trained, resulting in poor performance across most benchmarks. Conversely, when employing only local features for \abbr, performance improves markedly. This improvement can be attributed to the model's dedicated learning of local feature compression, resulting in well-trained local features. Despite this improvement, utilizing only local features proves insufficient across benchmarks, underscoring the crucial of global view.

\textbf{Effect of Number of Reserve Tokens.} Additionally, we validate our hypothesis that more image tokens do not always yield superior results. For instance, when $\gamma$ is set to $75\%$, consistent performance gains are evident across most benchmarks. This indicates that by discarding irrelevant image tokens and padding tokens, the model can focus more on those most pertinent to the posed question.

\textbf{Ablation Studies.} In Table \ref{tab:ablation}, we provide a comprehensive analysis of the effectiveness of each component within \abbr. Through the utilization of various adapters for global features, we have witnessed a notable improvement in image understanding, resulting in enhanced performance across a range of benchmarks. This underscores the significance of leveraging complex adapters for better extraction of global information. The inclusion of a detailed local compression layer has proven particularly advantageous for tasks demanding intricate visual analysis, such as TextVQA. Increasing $N_q$ from $64$ to $144$ further amplifies performance. However, as the number of local feature tokens increases, performance does not consistently rise. This is attributed to the overshadowing of global context by the abundance of local tokens, as evidenced by performance dips across most benchmarks. Notably, excessive local detail can also hinder performance gains in VQA$^T$. Therefore, we adopt 144 as our default setting. Lastly, the proposed SMR data demonstrates a significant enhancement in the model's reasoning ability, especially evident in mathematical and scientific datasets, while mitigating hallucination. Importantly, rigorous checks have been conducted on all training data to ensure there are no concerns regarding data leakage between training and test instances.

\begin{wrapfigure}{r}{0.72\linewidth}
\vspace{-0.3cm}
  \begin{center}
    \includegraphics[width=\linewidth]{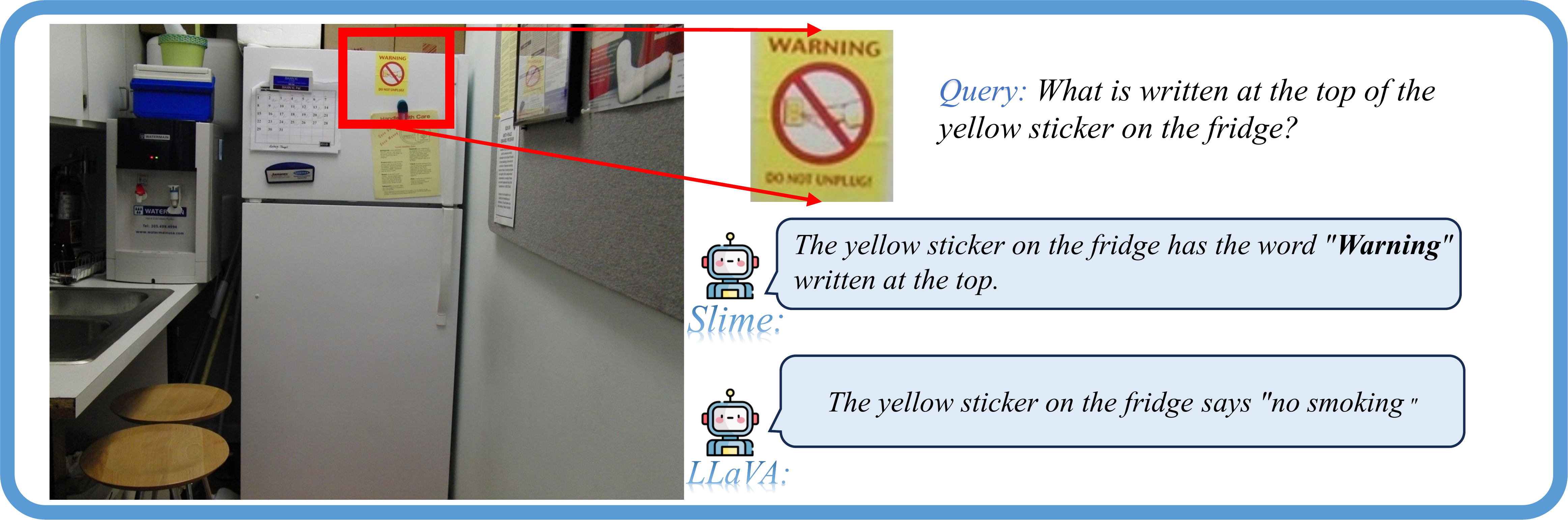}
\caption{\textbf{High-resolution image perception.}}
    \label{fig:hd_sample1}
  \end{center}
\vspace{-0.3cm}
\end{wrapfigure}

\textbf{Qualitative Results} of \abbr in high-resolution image perception are illustrated in Fig.~\ref{fig:hd_sample1}, and Figs.~\ref{fig:detail3} to \ref{fig:detail3}. These results emphasize the importance of local features to \abbr, as they enrich image details and facilitate a deeper understanding of vision information. Moreover, the ultimate \abbr demonstrates robust open-generation capabilities, including tasks such as code generation from flowcharts (Fig.~\ref{fig:detail_code}), story creation based on images (Fig.~\ref{fig:detail_story}), and providing suggestions (Fig.~\ref{fig:detail_sugg}).



\vspace{-0.2cm}
\section{Conclusion and Discussion}\label{sec:limit}
\vspace{-0.1cm}
In this paper, we delve into elucidating the intricacies of designing large multimodal models, with a specific focus on high-resolution image processing. Unlike previous studies that treat both the global view and sliced local image patches indiscriminately, our approach involves projecting and extracting global context using a mixture of experts, all without any feature compression. This methodology is rooted in the belief that global context encapsulates the majority of image information and holds greater significance than local patches. Local features undergo compression and selection based on their relevance to the query, thereby mitigating computation costs. Although training the framework end-to-end initially yields subpar performance, we address this by formulating the problem into a bi-level formulation and employing an alternating training strategy. This strategic maneuver circumvents optimization dilemmas inherent in end-to-end training. Our framework, dubbed \abbr, demonstrates promising performance across more than 10 benchmarks and even matches the performance of proprietary LMMs trained on significantly larger datasets, all with only 2 million training data points.

\textbf{Limitation and Future Work.} One main limitation lies in the 3-stage training approach. While alternating training proves superior to E2E training both theoretically and empirically, it inevitably extends the training duration. A promising avenue for improvement involves delving deeper into optimization methods tailored for such a bilinear formulation \cite{zhao2024improving}, potentially converting the alternating training strategy into a soft constraint within the gradient during E2E training. Another fruitful direction for future research is image token reduction. Given that existing studies consolidate all local and global features into LLMs, the computational cost for processing very-high-resolution images becomes prohibitively high. Therefore, an open question remains: can we further reduce image tokens, drawing inspiration from techniques such as token merging \cite{bolya2022tome} in computer vision? By doing so, we may preserve sufficient local details without necessitating additional image tokens for LLMs.

\bibliographystyle{plain}
\bibliography{neurips_2024}

\clearpage
\newpage
\appendix





\section{Related Work}\label{sec:related_work}
\textbf{Multimodal Large Language Models} have undergone significant evolution, initially rooted in BERT-based language decoders and later incorporating advancements in LLMs. Leveraging advanced LLMs such as GPTs~\citep{gpt4,brown2020language}, PaLM~\citep{chowdhery2023palm,anil2023palm}, BLOOM~\citep{muennighoff2022crosslingual}, LLaMA~\citep{touvron2023llama,touvron2023llama2}, Alpaca~\citep{taori2023stanford}, Vicuna~\citep{chiang2023vicuna}, and Mistral~\citep{jiang2023mistral}, Multimodal Large Language Models (LVLMs) exhibit enhanced capabilities and performance, particularly through end-to-end training techniques. Recent model developments, including Flamingo~\citep{awadalla2023openflamingo}, PaLI~\citep{laurenccon2024obelics}, PaLM-E~\citep{driess2023palm}, BLIP-2~\citep{li2023blip}, InstructBLIP~\citep{dai2024instructblip}, Otter~\citep{li2023otter}, MiniGPT-4~\citep{zhu2023minigpt}, mPLUG-Owl~\citep{ye2023mplug}, LLaVA~\citep{liu2023visual}, and QWen-VL~\citep{bai2023qwen}, bring unique perspectives to challenges such as scaling pre-training, enhancing instruction-following capabilities, and overcoming alignment issues. Our work is built upon LLaVA~\citep{liu2023visual}, but enhances all training datasets, model architecture, and alignment strategies, achieving state-of-the-art performance among existing LLMs.

\textbf{Advancements in Visual Instruction Tuning}: The efficacy of multimodal models heavily relies on the availability of high-quality image-text data for fine-tuning, a process known as visual instruction tuning~\citep{liu2023visual}. Previous studies have highlighted the limitations of constructing training sets solely based on existing Visual Question Answering (VQA) datasets~\citep{hudson2019gqa}, often resulting in degraded model performance. In an effort to address this, MiniGPT-4~\citep{zhu2023minigpt} meticulously curated 3,500 high-quality image-text pairs through a refinement process using ChatGPT, leading to more natural and reliable responses post-fine-tuning. In a pioneering initiative, LLaVA~\citep{liu2023visual} systematically constructed the LLaVA-Instruct-150K dataset for visual instruction tuning. Employing GPT-4, they generated questions and answers by providing image-level captions and object bounding boxes from the COCO dataset~\citep{lin2014microsoft}. To delve deeper into text-rich images, LLaVAR~\citep{zhang2023llavar} collected 422K noisy instruction-following instances using Optical Character Recognition (OCR) results and 16K high-quality instances using GPT-4. InstructBLIP~\citep{dai2024instructblip} amalgamated 26 public datasets, including LLaVA-Instruct-150K, to construct visual instruction tuning data. Innovatively, SVIT~\citep{zhao2023svit} constructed a comprehensive dataset comprising 4.2 million visual instruction tuning instances, including conversation question-answer pairs, complex reasoning QA pairs, referring QA pairs, and detailed image descriptions. However, many of these public datasets predominantly focus on visual perception and image captioning. In this study, we elevate the complexity of the Visual Instruction Tuning process by incorporating nine diverse datasets, encompassing scientific questions, mathematical/chart reasoning tasks, and even full reasoning paths. This augmentation aims to enhance LMMs to achieve significantly improved reasoning capabilities.


\section{Visual instruction tuning datasets}\label{app:datasets}

\begin{figure}
    \centering
    \includegraphics[width=\linewidth]{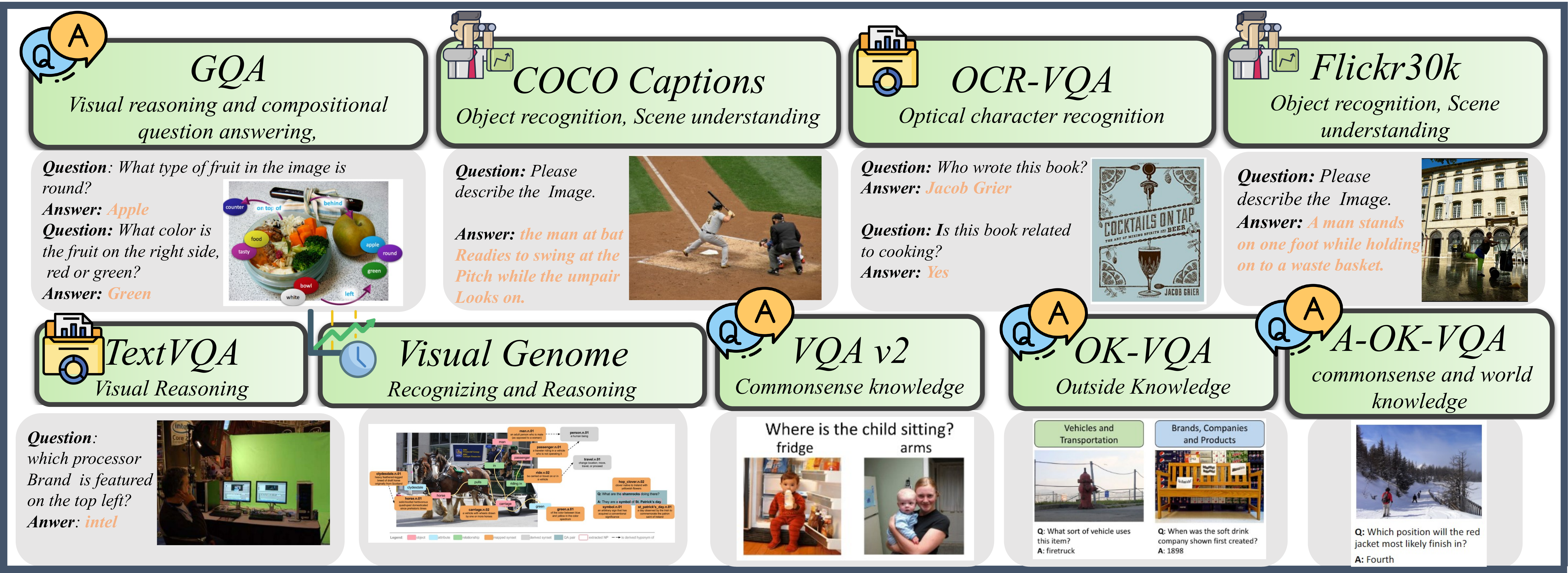}
    \caption{\textbf{Visual perception and reasoning tasks} in existing visual instruction tuning datasets.}
    \label{fig:vqa_dataset}
\end{figure}

Arxiv-QA~\citep{li2024multimodal}\footnote{\url{https://huggingface.co/datasets/MMInstruction/ArxivQA}} is a dataset featuring diverse images from scientific domains. Leveraging GPT-4V, the authors generated instruction-tuning datasets for generating QA pairs based on figures extracted from scientific papers. After filtering out invalid samples, the dataset consists of 100K QA pairs. The questions in the dataset have an average word count of 16.98, while the options for each question have an average word count of 31.86. On average, there are 4.20 options per question. For ArxivQA, we utilize rationales as answers for supervised fine-tuning.  


ScienceQA~\citep{lu2022learn}\footnote{\url{https://huggingface.co/datasets/derek-thomas/ScienceQA}} is a dataset characterized by rich domain diversity across three subjects: natural science, language science, and social science. Questions within each subject are categorized first by topic (e.g., Biology, Physics, Chemistry) and then further categorized by subtopics (e.g., Plants, Cells, Animals) and skills (e.g., Classify fruits and vegetables as plant parts, Identify countries of Africa). ScienceQA encompasses 26 topics, 127 categories, and 379 skills, providing comprehensive coverage across various domains. For ScienceQA, we utilize both the question and the instructions on how to solve the question as the prompt.

MATH-Vision~\citep{wang2024measuring}\footnote{\url{https://huggingface.co/datasets/mathvision/mathvision}} is an intricately assembled compilation of 3,040 meticulously selected mathematical problems accompanied by visual contexts sourced from authentic math competitions. Encompassing 16 distinct mathematical disciplines and graded across 5 levels of difficulty, the dataset offers a comprehensive and diverse array of challenges to assess the mathematical reasoning capabilities of Large Multimodal Models (LMMs).

AI2D~\citep{kembhavi2016diagram}\footnote{\url{https://github.com/allenai/dqa-net}} AI2 Diagrams (AI2D) is a dataset of over 5000 grade school science diagrams with over 150000 rich annotations, their ground truth syntactic parses, and more than 15000 corresponding multiple choice questions.

TextBookQA~\citep{kembhavi2017you}\footnote{\url{https://allenai.org/data/tqa}} is derived from middle school science curricula and comprises 1,076 lessons extracted from Life Science, Earth Science, and Physical Science textbooks. The dataset encompasses a total of 26,260 questions, with 12,567 of them accompanied by diagrams. We leverage the instructional content of images with the prompt that chosen from \cref{tab:image_commands} and consider questions with accompanying diagrams necessary for answering as VQA tasks.

GeoQA3~\citep{chen2021geoqa}\footnote{\url{https://github.com/chen-judge/GeoQA}} is a dataset comprising 4,998 diverse real-world geometric problems sourced from Chinese middle school exams. Each problem is further annotated with specific programs that describe the problem-solving process. The dataset encompasses three main problem types: angle calculation, length calculation, and others, which include various types of problems such as area calculation.

The Geometry3K Dataset\cite{lu2021inter}\footnote{\url{https://github.com/lupantech/InterGPS}} comprises 3,002 SAT-style problems sourced from two high-school textbooks covering diverse graph and goal types. Additionally, each problem in Geometry3K is annotated with dense descriptions in formal language.

Tabular Math Word Problems (TabMWP)\cite{lu2022dynamic}\footnote{\url{https://github.com/lupantech/PromptPG}} introduces a novel dataset containing 38,431 open-domain grade-level problems requiring mathematical reasoning on both textual and tabular data. Each question in TabMWP is associated with a tabular context, presented in the form of an image, semi-structured text, and a structured table. Questions are categorized into free-text and multi-choice types, with each problem annotated with gold solutions elucidating the multi-step reasoning process.

Data Visualization Question Answering (DVQA)\cite{kafle2018dvqa}\footnote{\url{https://github.com/kushalkafle/DVQA_dataset}} is a dataset designed to assess various aspects of bar chart understanding in a question answering framework. Unlike Visual Question Answering (VQA), DVQA necessitates processing words and answers unique to a particular bar chart. DVQA facilitates the automatic querying of extensive repositories of charts within scientific documents, web pages, and business reports. We randomly selected 100K data from the reasoning splits for our experiments.

ChartVQA\cite{masry2022chartqa}\footnote{\url{https://github.com/vis-nlp/ChartQA}} serves as a benchmark for question answering about charts with visual and logical reasoning. We exclusively selected human-authored QA pairs and excluded all machine-generated question-answer pairs to ensure data quality, resulting in a total of 9.6K question-answer pairs.

\begin{table}[ht]
  \centering
  \caption{\textbf{The list of instructions for detailed image description.}}
  \label{tab:image_commands}
  \begin{tabular}{cp{10cm}}
    \toprule \rowcolor{sclgreyblue!25}
    \textbf{Index} & \textbf{Description} \\ 
    \hline
    1 & Describe the following image in detail \\
    2 & Provide a detailed description of the given image \\
    3 & Give an elaborate explanation of the image you see \\
    4 & Share a comprehensive rundown of the presented image \\
    5 & Offer a thorough analysis of the image \\
    6 & Explain the various aspects of the image before you \\
    7 & Clarify the contents of the displayed image with great detail \\
    8 & Characterize the image using a well-detailed description \\
    9 & Break down the elements of the image in a detailed manner \\
    10 & Walk through the important details of the image \\
    11 & Portray the image with a rich, descriptive narrative \\
    12 & Narrate the contents of the image with precision \\
    13 & Analyze the image in a comprehensive and detailed manner \\
    14 & Illustrate the image through a descriptive explanation \\
    15 & Examine the image closely and share its details \\
    16 & Write an exhaustive depiction of the given image \\
    \bottomrule
  \end{tabular}
\end{table}

\section{Proof of Theoretical Statements}\label{sec:app_proof}

\subsection{Gradient Descent for Optimizing Bilinear Problem}
Let $u_0$ and $v_0$ be the initial solution and is given in the following form
\begin{eqnarray}
u_0 = \alpha_0 a + \beta_0 b, \; v_0 = \beta_0 a + \alpha_0 b
\label{eqn:init}
\end{eqnarray}
where $\alpha_0, \beta_0 \in \R$ are two scales. Here, we use the fact that $u$ and $v$ have to lie in the subspace spanned by $a$ and $b$. Using the gradient descent method, we update the solution $u_t$ and $v_t$ as
\[
u_{t+1} = u_t - \eta\left(u_t v_t^{\top} - X\right) v_t, \quad v_{t+1} = v_t - \eta\left(v_t u_t^{\top} - X\right)u_t
\]
The following theorem describes the dynamics of $u$ and $v$ over the iterations of gradient descent. 
\begin{thm}
Define $z_t = (\alpha_t, \beta_t)^{\top}$. Under the initialization of $u$ and $v$ given in Eq.~\ref{eqn:init}, we have, for all $t \geq 0$
\[
u_t = \alpha_t a + \beta_t b, \; v_0 = \beta_t a + \alpha_t b
\]
with $|u_t|^2 = |v_t|^2$, where
\begin{eqnarray}
z_{t+1} = \underbrace{\left(\begin{array}{cc}
1 + \eta(1 - |u_t|^2) & \eta a^{\top}b \\
\eta a^{\top}b & 1 + \eta(1 - |u_t|^2)
\end{array}\right)}_{:= F_t}z_t \label{eqn:update}
\end{eqnarray}
\end{thm}
\begin{proof}
We prove the result by induction. It is easy to see that the theorem hold for $t=0$. We will then assume it is true for the case of $t$ and show it also holds for $t+1$. We first have $u_{t+1}$ as
\begin{eqnarray*}
u_{t+1} & = & u_t - \eta\left(u_tv_t^{\top} - X\right)v_0 \\
& = & u_t\left(1 - \eta |v_t|^2\right) + \eta Xv_t \\
& = & u_t\left(1 - \eta |v_t|^2\right) + \eta \left(ab^{\top} + ba^{\top}\right)(\beta_t a + \alpha_t b) \\
& = & u_t\left(1 - \eta |u_t|^2\right) + \eta\left(\beta_t b^{\top}a a + \beta_t b + \alpha_t a + \alpha_ta^{\top}b b\right) \\
& = & \left((1 - \eta |u_t|^2 + \eta)\alpha_t + \eta a^{\top}b\beta_t\right)a + \left((1 - \eta|u_t|^2 + \eta)\beta_t + a^{\top}b \alpha_t\right)b
\end{eqnarray*}
where we utilize the assumption $|v_t| = |u_t|$. 
Hence, by writing $u_{t+1} = \alpha_{t+1} a + \beta_{t+1} b$, we have
\[
\left(\begin{array}{c}
\alpha_{t+1} \\
\beta_{t+1}
\end{array}
\right) = \left(\begin{array}{cc}
1 + \eta(1 - |u_t|^2) & \eta a^{\top}b \\
\eta a^{\top}b & 1 + \eta(1 - |u_t|^2)
\end{array}\right)\left(\begin{array}{c}
\alpha_t \\
\beta_t
\end{array}
\right)
\]
Similarly, we apply gradient descent to update $v_t$ as follows
\begin{eqnarray*}
v_{t+1} & = & v_t - \eta\left(v_t u_t^{\top} - X\right)u_t \\
& = & v_t\left(1 - \eta |u_t|^2\right) + \eta Xu_t \\
& = & v_t\left(1 - \eta |u_t|^2\right) + \eta \left(ab^{\top} + ba^{\top}\right)(\alpha_t a + \beta_t b) \\
& = & v_t\left(1 - \eta |v_t|^2\right) + \eta\left(b^{\top}a \alpha_t a + \alpha_t b + \beta_t a + a^{\top}b\beta_t b\right) \\
& = & \left((1 - \eta |v_t|^2 + \eta)\beta_t + \eta a^{\top}b\alpha_t\right)a + \left((1 - \eta|v_t|^2 + \eta)\alpha_t + a^{\top}b \beta_t\right)b
\end{eqnarray*}
Thus, it is easy to see that $v_{t+1} = \beta_{t+1} a + \alpha_{t+1} b$. 
\end{proof}

Next, we will show that the updating in Eq.~\ref{eqn:update} will not converge the largest eigenvector of $M$. Let $\lambda_+ > \lambda_-$ denote the largest and smallest eigenvalues of $M$, and let $w_+$ and $w_-$ denote the corresponding eigenvectors, respectively. We have
\[
\lambda_+ = 1 + a^{\top}b, \quad \lambda_- = 1 - a^{\top}b
\]
and
\[
w_+ = \frac{1}{\sqrt{2}}\left(
\begin{array}{c}
1 \\
1
\end{array}
\right), \quad w_- = \frac{1}{\sqrt{2}}\left(
\begin{array}{c}
1 \\
-1
\end{array}
\right)
\]
Let $\tau_t = z_t^{\top} w_+$ and $\nu_t = z_t^{\top}w_-$. Using these notation, we can write $|u_t|^2$ as
\begin{eqnarray}
|u_t|^2 = z_t^{\top}M z_t = \left(1 + a^{\top}b\right)\tau_t^2 + \left(1 - a^{\top}b\right)\nu_t^2
\label{eqn:norm}
\end{eqnarray}
The following lemma describes how $\tau_t$ and $\nu_t$ evolves over iterations.
\begin{lemma} We have
\begin{eqnarray}
\tau_{t+1} = \left(1 + \eta(1 + a^{\top}b - |u_t|^2)\right)\tau_t, \quad \nu_{t+1} = \left(1 + \eta(1 - a^{\top}b - |u_t|^2)\right)\nu_t  \label{eqn:update-1}
\end{eqnarray}
\end{lemma}
\begin{proof}
According to Eq.~\ref{eqn:update}, we can write 
\[
F_t = (1 - \eta|u_t|^2) I + \eta M
\]
and therefore can decompose it into the following form
\[
F_t = (1 - \eta|u_t|^2 + \eta \lambda_+)w_+w_+^{\top} + (1 - \eta|u_t|^2 + \eta \lambda_-)w_-w_-^{\top}
\]
Since $z_t = \tau_+w_+ + \nu_t w_-$, we have
\[
F_t z_t = (1 - \eta|u_t|^2 + \eta \lambda_+)\tau_tw_+ + (1 - \eta|u_t|^2 + \eta \lambda_-)\nu_tw_-
\]
and therefore
\[
\tau_{t+1} = \left(1 + \eta(1 + a^{\top}b - |u_t|^2)\right)\tau_t, \quad \nu_{t+1} = \left(1 + \eta(1 - a^{\top}b - |u_t|^2)\right)\nu_t
\]
\end{proof}
When $\eta$ is very small, we can write the updating equation in Eq.~\ref{eqn:update-1} as differential equations, i.e.
\begin{eqnarray}
\frac{d \tau_t}{d t} = \left(1 + a^{\top}b - |u_t|^2\right) \tau_t, \quad \frac{d \nu_t}{d t} = \left(1 - a^{\top}b - |u_t|^2\right) \nu_t \label{eqn:diff}
\end{eqnarray}
The following reveals the convergence property of Eq.~\ref{eqn:diff}.

\begin{thm}
ODE in Eq.~\ref{eqn:diff} have two converged solutions, with one being $\tau_t \rightarrow 1, \nu_t \rightarrow 0$ and the other being $\tau_t \rightarrow 0, \nu_t \rightarrow 1$ when $t \rightarrow \infty$
\label{theo:4}
\end{thm}
\begin{proof}
Using the expression of $|u_t|^2$ in Eq.~\ref{eqn:norm}, we have
\begin{eqnarray*}
\frac{d|u_t|^2}{d t} & = & 2\left(1 + a^{\top}b\right)\tau_t\frac{d\tau_t}{d t} + 2\left(1 - a^{\top}b\right)\nu_t \frac{d \nu_t}{d t} \\
& = & 2\left(1 + a^{\top}b\right)\left(1 + a^{\top}b - |u_t|^2\right)\tau_t^2 + 2\left(1 - a^{\top}b\right)\left(1 - a^{\top}b - |u_t|^2\right)\nu_t^2 \\
& = & 2(1 - |u_t|^2) |u_t|^2 + 2a^{\top}b\left((1 + a^{\top}b)\tau_t^2 - (1 - a^{\top}b)\nu_t^2\right)
\end{eqnarray*}
When the ODE converges to a fix point $u_*$, we should have
\begin{eqnarray}
2(1 - |u_*|^2) |u_*|^2 + 2a^{\top}b\left((1 + a^{\top}b)\tau_*^2 - (1 - a^{\top}b)\nu_*^2\right) = 0 \label{eqn:cond}
\end{eqnarray}
At the same time, according to Eq.~\ref{eqn:diff}, when ODE converges, we should have either 
\[
|u_*|^2 = 1 + a^{\top}b, \quad v_* =0
\] or 
\[
|u_*|^2 = 1 - a^{\top}b, \quad \tau_* = 0.
\]
We complete the proof by combining these two conditions with the condition in Eq.~\ref{eqn:cond}. 
\end{proof}
As indicated by Theorem~\ref{theo:4}, the iterative updating by gradient descent does not necessarily converge to the optimal solution, i.e., $\tau_* = 1$ and $\nu_* = 0$, rendering the simultaneous updating of $u$ and $v$ less ideal for optimizing the objective function of the bilinear form.

\subsection{Optimizing Bilinear Problem by Alternating Optimization}
We will demonstrate that the issue with gradient descent (or more precisely, simultaneously updating $u$ and $v$) can be effectively addressed through alternating optimization. Specifically, we will optimize $v$ with $u$ fixed, and then optimize $u$ with $v$ fixed. We will illustrate that this approach leads to convergence towards the optimal solution. Let $u_0 = \alpha_0 a + \beta_0 b$. We denote the sequential solution $u_t$ obtained through alternating optimization as $u_t = \alpha_t a + \beta_t b$. The following theorem describes the evolution of $z_t = (\alpha_t, \beta_t)^{\top}$ over iterations.

\begin{thm}
We have
\[
z_{t+1} = \frac{1}{|u_t|^2|v_t|^2}M^2 z_t
\]
\label{theo:5}
\end{thm}
\begin{proof}
Let $u_t = \alpha_t a + \beta_t b$. By fixing $u_t$, we have the optimal solution for $v_t$
\begin{eqnarray*}
v_t & = & \frac{1}{|u_t|^2} Xu_t \\
& = & \frac{1}{|u_t|^2}\left(a^{\top}b \alpha_t a + \alpha_t b + \beta_t a + a^{\top}b \beta_t b\right)
\end{eqnarray*}
By writing $v_t = \beta_t' a + \alpha_t' b$, we have
\[
\left(\begin{array}{c}
\alpha_t' \\
\beta_t'
\end{array}
\right) = \frac{1}{|u_t|^2}\left(\begin{array}{cc}
1 & a^{\top}b \\
a^{\top}b & 1
\end{array}\right)\left(\begin{array}{c}
\alpha_t \\
\beta_t
\end{array}
\right)
\]
We then fix $v_t$ and find the optimal solution for $u_{t+1}$, i.e.
\begin{eqnarray*}
u_{t+1} & = & \frac{1}{|v_t|^2} Xv_t \\
& = & \frac{1}{|v_t|^2}\left(a^{\top}b \beta_t' a + \beta_t' b + \alpha_t' a + a^{\top}b \alpha_t' b\right)
\end{eqnarray*}
By writing $u_{t+1} = \alpha_{t+1} a + \beta_{t+1} b$, we have
\[
\left(\begin{array}{c}
\alpha_{t+1} \\
\beta_{t+1}
\end{array}
\right) = \frac{1}{|v_t|^2}\left(\begin{array}{cc}
1 & a^{\top}b \\
a^{\top}b & 1
\end{array}\right)\left(\begin{array}{c}
\alpha_t' \\
\beta_t'
\end{array}
\right) = \frac{1}{|u_t|^2|v_t|^2}\left(\begin{array}{cc}
1 & a^{\top}b \\
a^{\top}b & 1
\end{array}\right)^2\left(\begin{array}{c}
\alpha_t \\
\beta_t
\end{array}
\right)
\]
\end{proof}
According to Theorem~\ref{theo:5}, by alternating optimization, we have
\[
z_{t} \propto M^{2t}z_0
\]
implying that $z_t$ will guarantee to converge to the largest eigenvector of $M$, which resolves the limitation of gradient descent. 

\section{Experiments}\label{sec:exp_app}

\subsection{Experimental Details}

\texttt{\color{sclgreyblue}General QA and Open-ended Generation Benchmarks}
The MMMU dataset~\citep{yue2023mmmu} comprises meticulously curated multimodal questions totaling 11.5K, sourced from college exams, quizzes, and textbooks. Spanning six core disciplines—Art and Design, Business, Science, Health and Medicine, Humanities and Social Science, and Tech and Engineering—it covers 30 subjects and 183 subfields. With a focus on advanced perception and reasoning embedded with domain-specific knowledge, MMMU challenges models to tackle tasks akin to those encountered by experts. It incorporates 30 highly diverse image types, including charts, diagrams, maps, tables, music sheets, and chemical structures, diverging from existing benchmarks by emphasizing sophisticated perception and reasoning. The MME dataset~\citep{fu2023mme} distinguishes itself by concurrently evaluating perception and cognition capabilities, encompassing tasks such as OCR, coarse-grained object recognition (including existence, count, position, and color), and fine-grained object recognition (encompassing movie posters, celebrities, scenes, landmarks, and artworks). With a total of 14 subtasks, MME provides a comprehensive evaluation, catering to the need for a thorough assessment of MLLMs across diverse modalities and cognitive domains. Additionally, traditional benchmarks such as MMBench-EN~\citep{liu2024mmbench}, MMVet~\citep{yu2024mm}, GQA~\citep{hudson2018gqa}, VQA~\citep{Goyal_2017_CVPR}, and Text-VQA~\citep{8953586} were utilized to gauge the model's visual understanding and reasoning abilities. Further evaluations included testing the model's proficiency in Chinese and its understanding of Chinese culture through the MMBench-CN test~\citep{liu2024mmbench}, along with assessing conversation abilities via LLaVA-Bench~\citep{liu2023visual}.

\texttt{\color{sclgreyblue}Math Reasoning and Science QA Benchmarks:} MathVista~\citep{lu2024mathvista} presents a multifaceted benchmark that integrates challenges from diverse mathematical and visual tasks. Comprising 6,141 examples, it amalgamates data from 28 existing multimodal datasets alongside the introduction of three novel datasets: IQTest, FunctionQA, and PaperQA. Solving these tasks demands intricate visual comprehension and compositional reasoning, posing significant challenges even for state-of-the-art foundation models. MathVision~\citep{wang2024measuring} meticulously curates a collection of 3,040 high-quality mathematical problems with visual contexts drawn from real math competitions. Encompassing 16 distinct mathematical disciplines and graded across 5 difficulty levels, this dataset offers a comprehensive and diverse range of challenges for assessing the mathematical reasoning abilities of LMMs. MathVerse provides an inclusive visual math benchmark by collecting 2,612 high-quality, multi-subject math problems with accompanying diagrams sourced from publicly available materials. Each problem undergoes transformation by human annotators, resulting in six distinct versions with varying levels of information content in multimodality, culminating in 15K test samples. This approach enables MathVerse to thoroughly evaluate the ability of LMMs to comprehend visual diagrams for mathematical reasoning. ScienceQA~\citep{lu2022learn} covers 26 topics, 127 categories, and 379 skills, offering comprehensive coverage across various domains.

\texttt{\color{sclgreyblue}{Hallucination}}   The existence and count subsets of the MME dataset~\citep{fu2023mme} are applicable for object-level hallucination, while the position and color subsets are suitable for attribute-level hallucination. POPE~\citep{li2023evaluating} introduces a streamlined method for evaluating object hallucination in LMMs. In this assessment framework, LMMs are tasked with determining the presence of a specific object in a given image. The benchmark ensures a balanced ratio between queries probing existent and non-existent objects (i.e., 50\% each). AMBER~\citep{wang2023llm} is an LLM-free Multi-dimensional Benchmark for LMMs hallucination evaluation, which can be used to evaluate fir the discriminative task including existence, attribute and relation hallucination.

We provide a thorough examination of the evaluation benchmarks employed, accompanied by their respective metrics as detailed in \cref{tab:benchmark_metrics}. Additionally, we outline the training hyperparameters for both the initial vision-language alignment pretraining and the subsequent visual instruction tuning stages in \cref{tab:hyper-parameter}.

\begin{table}[]
\caption{\textbf{Summary of the evaluation benchmarks}.}\label{tab:benchmark_metrics}
\resizebox{\textwidth}{!}{%
\begin{tabular}{@{}ccc@{}}
\toprule  \rowcolor{sclgreyblue!25}
\textbf{Data} & \textbf{Response formatting prompts} & \textbf{Metric} \\ \midrule
Text-VQA & Answer the question using a single word or phrase. & Accuracy ($\uparrow$) \\
GQA & Answer the question using a single word or phrase. & Accuracy ($\uparrow$) \\
VQA-v2 & Answer the question using a single word or phrase. & Accuracy ($\uparrow$) \\
MMBench, MMBench-CN & Answer with the option’s letter from the given choices directly. & Accuracy ($\uparrow$) \\
MMMU (Multi-choice) & Answer with the option’s letter from the given choices directly. & Accuracy ($\uparrow$) \\
MMMU (Short answer) & Answer the question using a single word or phrase. & Accuracy ($\uparrow$) \\
LLaVA-Bench & - & Score Ratio compared to GPT-4 ($\uparrow$) \\
MME & Answer the question using a single word or phrase. & Total Score ($\uparrow$)\\
Science-QA & Answer with the option’s letter from the given choices directly. & IMG-Accuracy ($\uparrow$) \\
MathVision & Answer the question using a single word or phrase. & Accuracy ($\uparrow$) \\
MathVista & - & Accuracy ($\uparrow$) \\
MathVerse & - & Accuracy ($\uparrow$) \\
POPE & Answer the question using a single word or phrase. & F1 score ($\uparrow$) \\
AMBER & Answer the question using a single word or phrase. & Accuracy ($\uparrow$) \\

\bottomrule
\end{tabular}%
}
\end{table}

\begin{table}[]
\centering
\caption{\textbf{Hyperparameters of \abbr}. The hyper-parameters of Vicuna-13B is similar to Vicuna-7B.}\label{tab:hyper-parameter}
\resizebox{0.8\textwidth}{!}{%
\begin{tabular}{@{}ccc|cc|cc@{}}
\toprule  \rowcolor{sclgreyblue!25}
 LLM & \multicolumn{2}{c}{Vicuna-7B} & \multicolumn{2}{c}{Llama3-8B} & \multicolumn{2}{c}{Llama3-70B (LoRA)}  \\ \midrule
 & Stage I \& II & Stage III & Stage I \& II & Stage III & Stage I \& II & Stage III \\
batch size & 256 & 128 & 256 & 128 & 256 & 128 \\
lr & 1.00E-03 & 2.00E-05 & 1.00E-03 & 1.00E-05 & 1.00E-03 & 1.00E-04 \\
Deepspeed ZeRO Stage & 2 & 3 & 2 & 3 & 3 & 3 offload \\
lr schedule & \multicolumn{6}{c}{Cosine Annealing with Linear Warmup} \\
lr warmup ratio & \multicolumn{6}{c}{0.03} \\
Weight Decay & \multicolumn{6}{c}{0} \\
Epoch & \multicolumn{6}{c}{1} \\
Optimizer & \multicolumn{6}{c}{AdamW} \\
Precision & \multicolumn{6}{c}{BF16} \\ \bottomrule
\end{tabular}%
}
\end{table}

\subsection{Extended Experiments}\label{sec:extend_exp}

\textbf{Instruction Following Concerns Regarding Fine-tuning with LoRA.} As illustrated in Fig.~\ref{fig:lora_example}, within the MathVerse benchmark, despite our prompt explicitly requesting the model to \textit{``Please directly answer the question and provide the correct option letter''}, \abbr-8B$^\dagger$ consistently generates the entire reasoning path. While this does not necessarily result in incorrect answers, it diverges from our intended approach. Intriguingly, when we modify the prompt to \textit{``Please directly answer the question \textbf{using a single word or phrase} and provide the correct option letter''}, the model responds with just "$C$," highlighting the sensitivity of \abbr-8B$^\dagger$. Additionally, we observed that the performance of \abbr-8B$^\dagger$ is inferior to that of \abbr-8B, as exemplified in Fig.~\ref{fig:lora2}. Thus, while LoRA fine-tuning maintains similar performance across most benchmarks, the constrained parameter updates and limited alterations to the parameters of LLMs may make it more challenging than fully fine-tuning to strictly adhere to instructions, particularly when image prompts are involved.
\begin{figure*}[t]
\subfigure[\textbf{MathVerse.}]{
\begin{minipage}[t]{0.4\linewidth}
\centering
 \includegraphics[width=\linewidth]{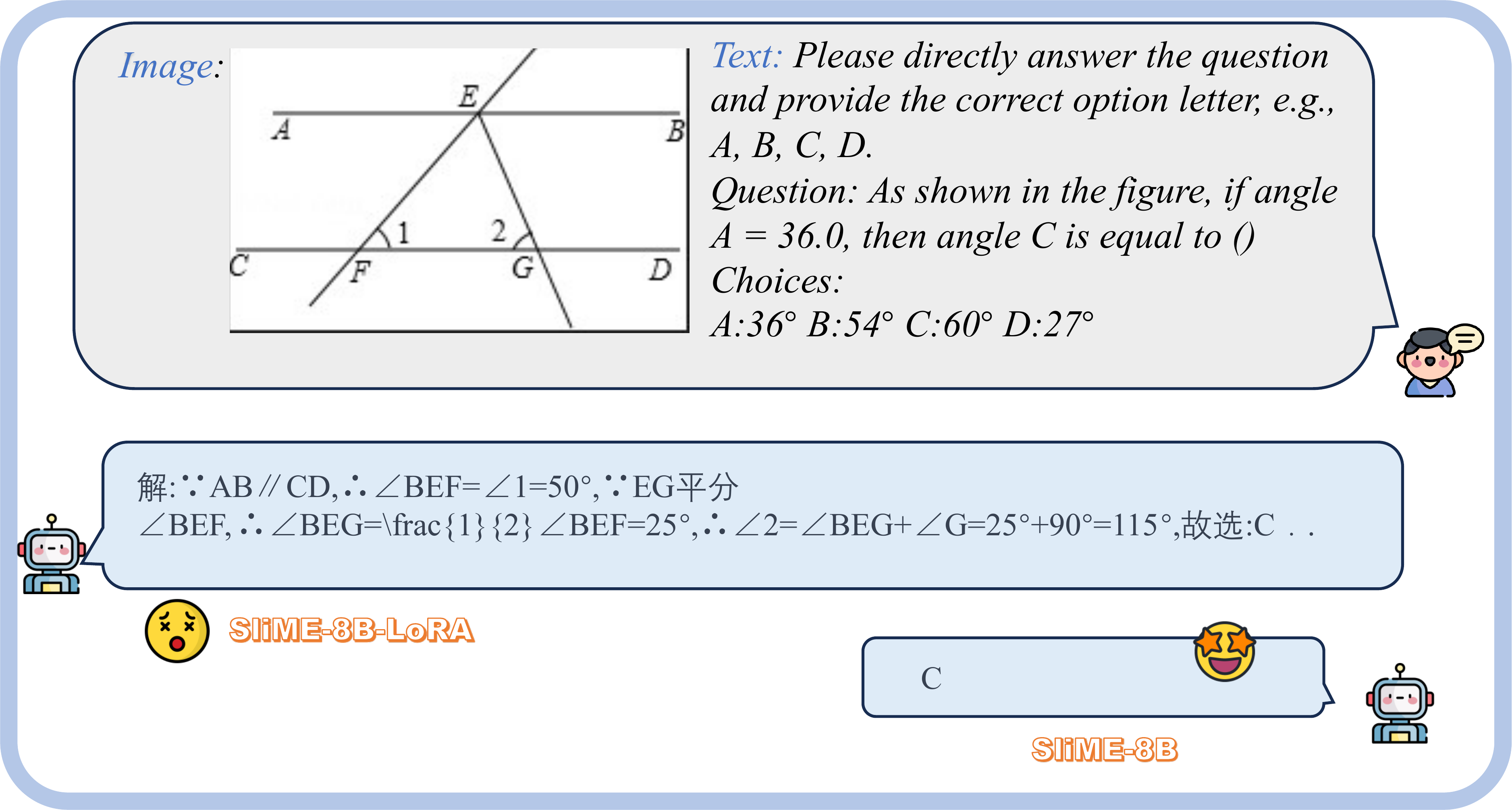}
    \label{fig:lora1}
\end{minipage}%
}%
\subfigure[\textbf{LLaVA-Bench.}]{
\begin{minipage}[t]{0.55\linewidth}
 \includegraphics[width=\linewidth]{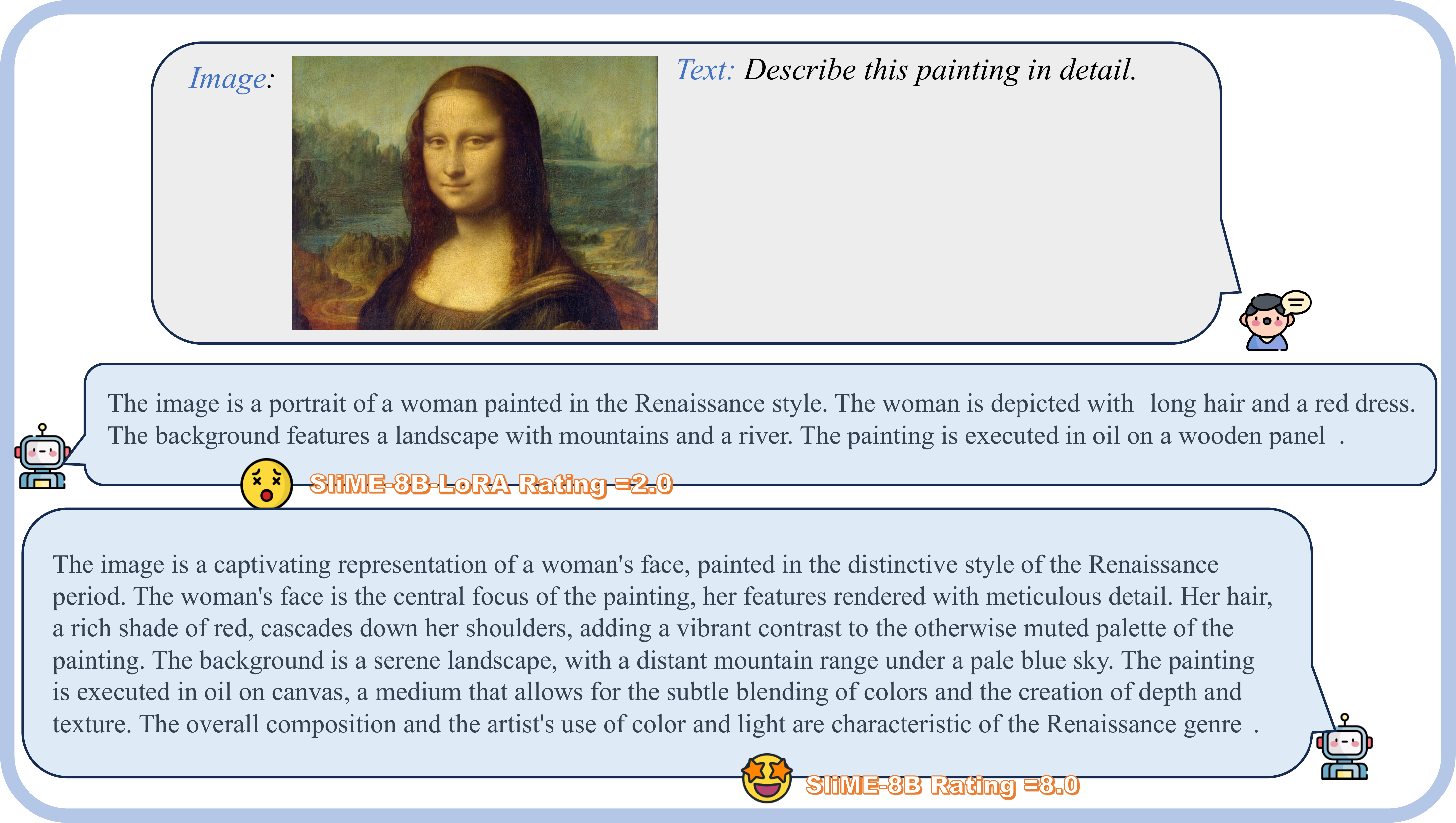}
    \label{fig:lora2}
\end{minipage}%
}%
\centering
\vspace{-0.2cm}
\caption{\textbf{Failure cases of \abbr-8B$^\dagger$.}}
\label{fig:lora_example}
\end{figure*}


\begin{figure}
    \centering
    \includegraphics[width=\linewidth]{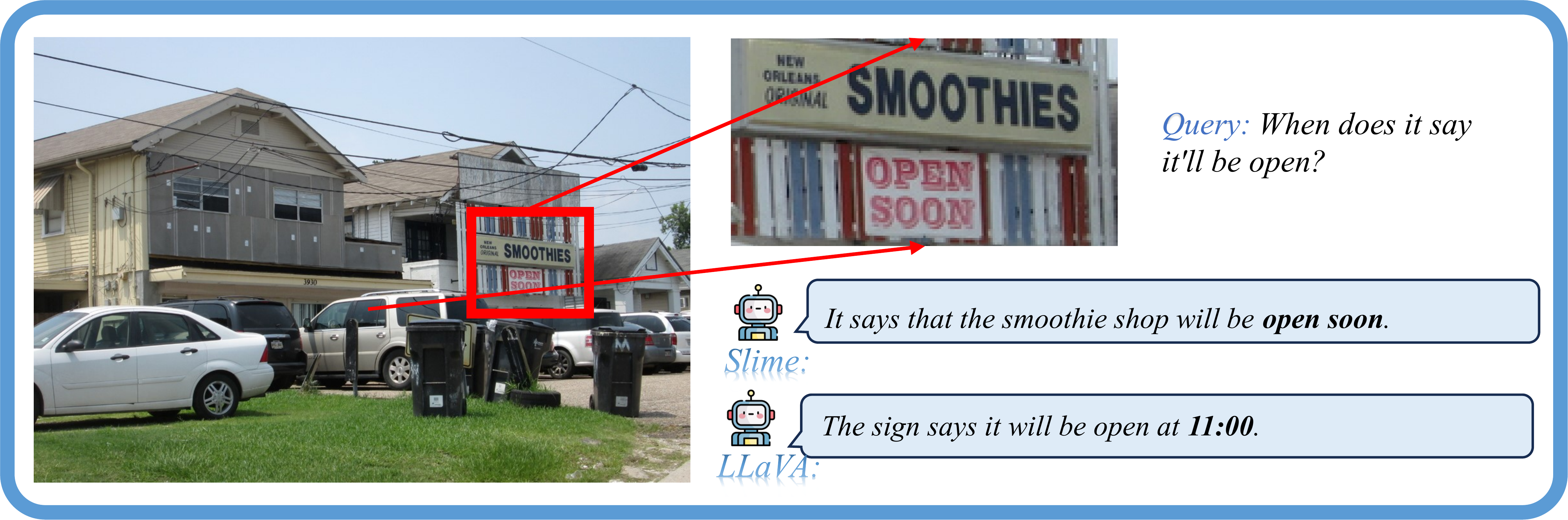}
    \caption{\textbf{High-resolution image perception.}}
    \label{fig:detail2}
\end{figure}

\begin{figure}
    \centering
\includegraphics[width=\linewidth]{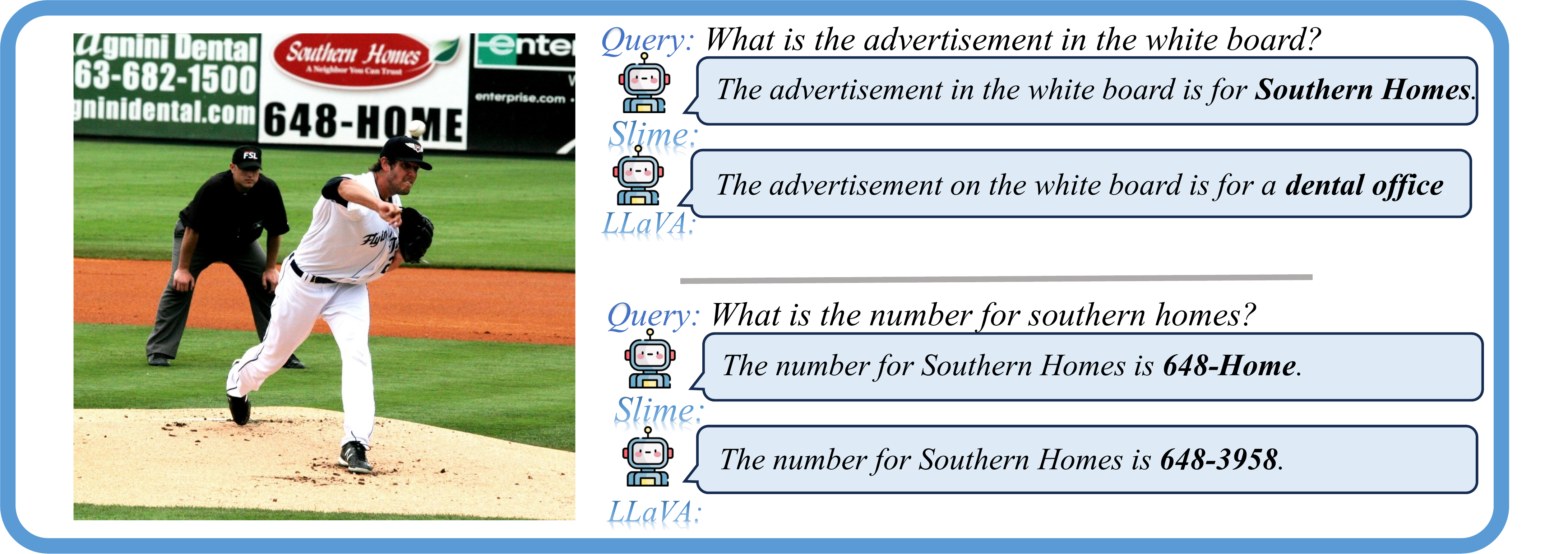}
    \caption{\textbf{High-resolution image perception.}}
    \label{fig:detail3}
\end{figure}

\begin{figure}
    \centering
    \includegraphics[width=\linewidth]{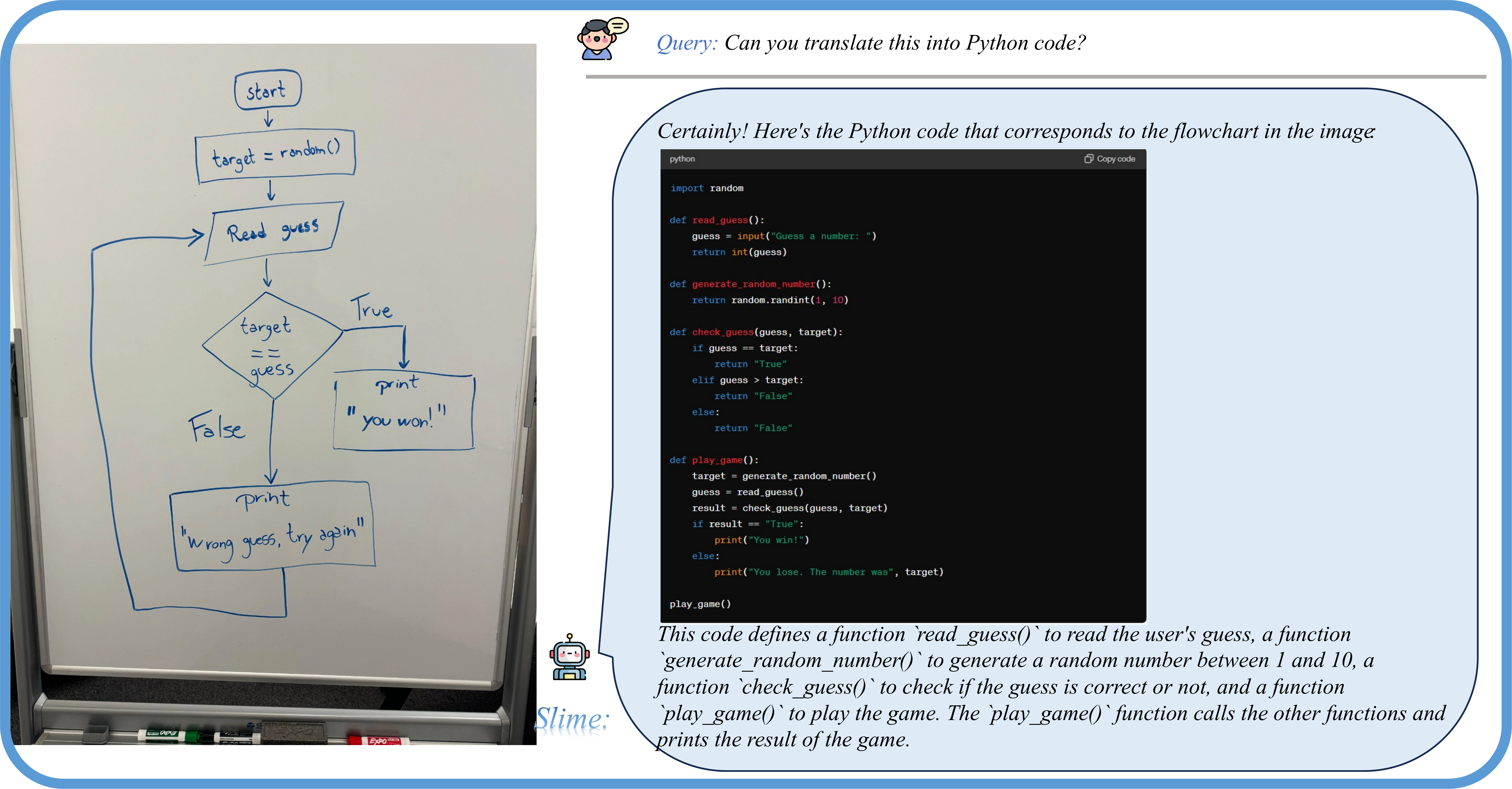}
    \caption{\textbf{Code generation.}}
    \label{fig:detail_code}
\end{figure}

\begin{figure}
    \centering
    \includegraphics[width=\linewidth]{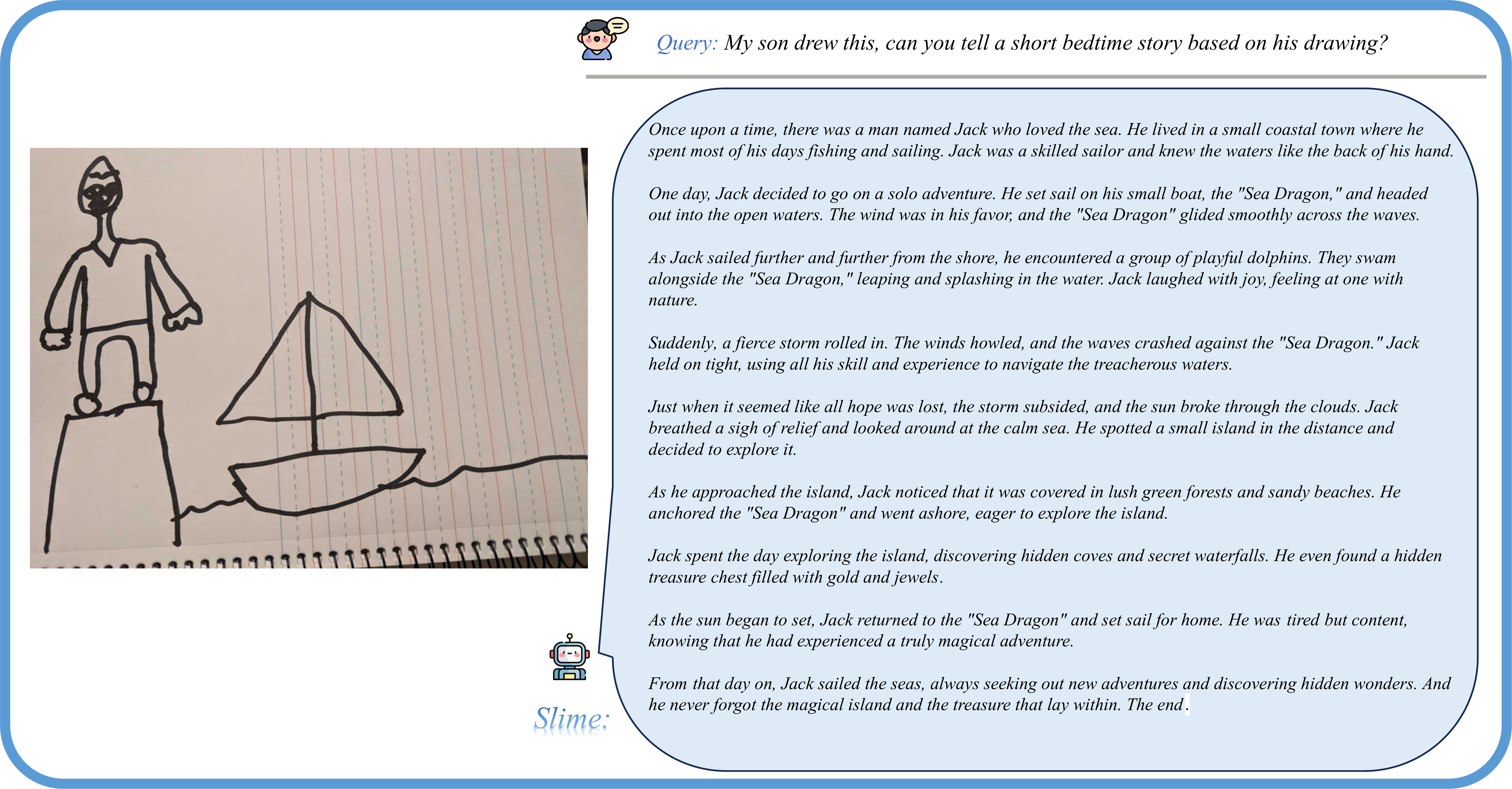}
    \caption{\textbf{Story writing.}}
    \label{fig:detail_story}
\end{figure}

\begin{figure}
    \centering
    \includegraphics[width=\linewidth]{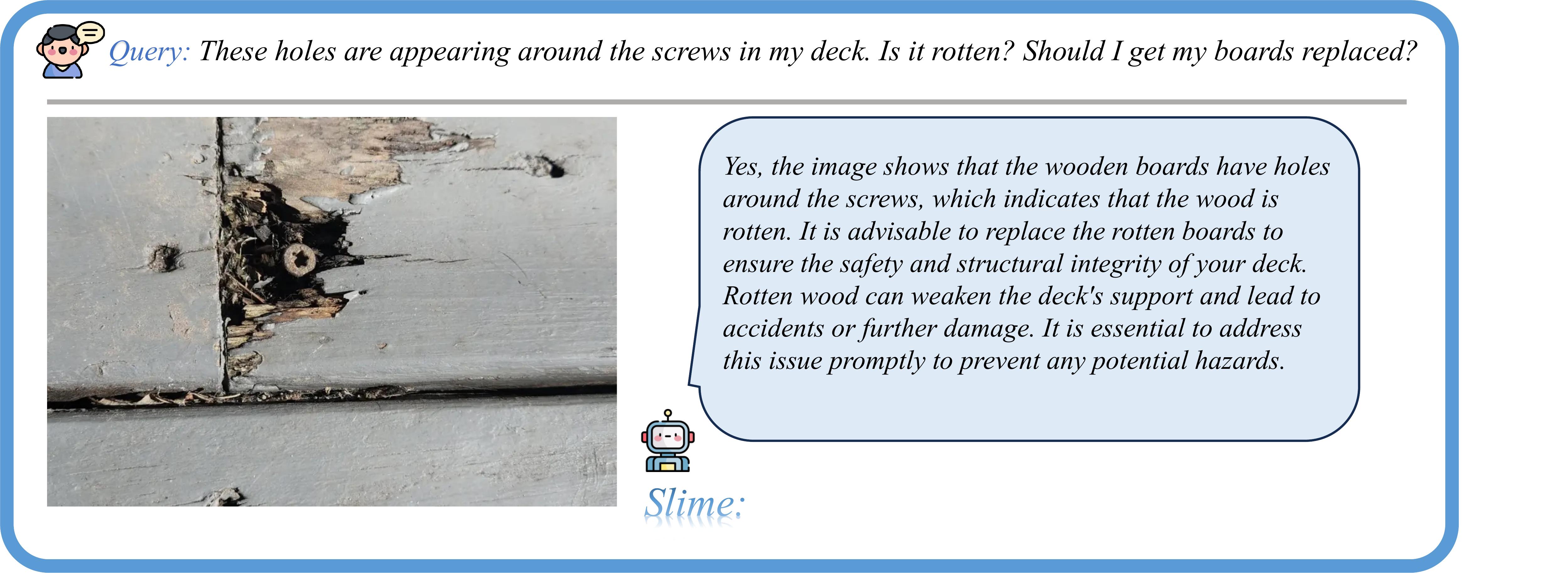}
    \caption{\textbf{Suggestion.}}
    \label{fig:detail_sugg}
\end{figure}

\end{document}